\documentclass{article}

\PassOptionsToPackage{numbers, compress}{natbib}

\usepackage[final]{neurips_2020}
\usepackage[utf8]{inputenc} 
\usepackage[T1]{fontenc}    
\usepackage{color,xcolor}
\definecolor{mydarkblue}{rgb}{0,0.08,0.45}
\usepackage[colorlinks=true,
    linkcolor=mydarkblue,
    citecolor=mydarkblue,
    filecolor=mydarkblue,
    urlcolor=mydarkblue]{hyperref}  
\usepackage{url}            
\usepackage{booktabs}       
\usepackage{amsfonts}       
\usepackage{nicefrac}       
\usepackage{microtype}      
\usepackage{subfigure}
\usepackage{amsmath,amssymb}
\usepackage{amsfonts}
\usepackage{amsthm}
\usepackage{xspace}
\usepackage{graphicx}

\usepackage{amsmath,amsfonts,bm}
\newtheorem{theorem}{Theorem}

\newtheorem{lemma}{Lemma}

\usepackage{colortbl}
\usepackage{pgfplots}
\usepackage{pgfplotstable}

\pgfplotstableset{
    /color cells/min/.initial=0,
    /color cells/max/.initial=1000,
    /color cells/textcolor/.initial=,
    color cells/.code={%
        \pgfqkeys{/color cells}{#1}%
        \pgfkeysalso{%
            postproc cell content/.code={%
                \begingroup
                \pgfkeysgetvalue{/pgfplots/table/@preprocessed cell content}\value
\ifx\value\empty
\endgroup
\else
                \pgfmathfloatparsenumber{\value}%
                \pgfmathfloattofixed{\pgfmathresult}%
                \let\value=\pgfmathresult
                \pgfplotscolormapaccess
                    [\pgfkeysvalueof{/color cells/min}:\pgfkeysvalueof{/color cells/max}]%
                    {\value}%
                    {\pgfkeysvalueof{/pgfplots/colormap name}}%
                \pgfkeysgetvalue{/pgfplots/table/@cell content}\typesetvalue
                \pgfkeysgetvalue{/color cells/textcolor}\textcolorvalue
                \toks0=\expandafter{\typesetvalue}%
                \xdef\temp{%
                    \noexpand\pgfkeysalso{%
                        @cell content={%
                            \noexpand\cellcolor[rgb]{\pgfmathresult}%
                            \noexpand\definecolor{mapped color}{rgb}{\pgfmathresult}%
                            \ifx\textcolorvalue\empty
                            \else
                                \noexpand\color{\textcolorvalue}%
                            \fi
                            \the\toks0 %
                        }%
                    }%
                }%
                \endgroup
                \temp
\fi
            }%
        }%
    }
}

\newcommand{\RNum}[1]{\uppercase\expandafter{\romannumeral #1\relax}}

\newcommand\restr[2]{{
  \left.\kern-\nulldelimiterspace 
  #1 
  \vphantom{\big|} 
  \right|_{#2} 
  }}

\DeclareMathOperator*{\argmax}{\arg\!\max}
\newcommand{\xhdr}[1]{{\noindent\bfseries #1}.}
\newcommand{\cut}[1]{}

\newcommand{\removelatexerror}{\let\@latex@error\@gobble}

\def\1{\bm{1}}

\DeclareMathAlphabet{\mathsfit}{\encodingdefault}{\sfdefault}{m}{sl}
\SetMathAlphabet{\mathsfit}{bold}{\encodingdefault}{\sfdefault}{bx}{n}

\def\gB{{\mathcal{B}}}

\def\gD{{\mathcal{D}}}

\def\gF{{\mathcal{F}}}

\def\gX{{\mathcal{X}}}
\def\gY{{\mathcal{Y}}}
\def\gZ{{\mathcal{Z}}}

\def\sF{{\mathbb{F}}}

\def\sP{{\mathbb{P}}}

\def\sR{{\mathbb{R}}}

\newcommand{\E}{\mathbb{E}}

\DeclareMathOperator{\sign}{sign}
\DeclareMathOperator{\supp}{supp}

\usepackage[bibliography=common]{apxproof}
\usepackage{appendix}
\usepackage{multirow}
\usepackage{multicol}
\usepackage{hyperref}
\usepackage{xcolor}
\usepackage{wrapfig}
\usepackage[inline, shortlabels]{enumitem}  

\renewcommand{\appendixprelim}{\clearpage 
\section*{\centering  \huge Adversarial Example Games \\ Supplementary Materials  }
}

\definecolor{medium-blue}{rgb}{0,0,1}

\newtheoremrep{thm}{Theorem}

\newtheoremrep{obs}{Observation}
\newtheoremrep{prop}{Proposition}
\newtheoremrep{definition}{Definition}
\newtheoremrep{remark}{Remark}

\newcounter{ggiCounter}

\renewcommand{\cite}[1]{\citep{#1}} 

\title{Adversarial Example Games}

\author{%
  Avishek Joey Bose\thanks{Equal Contribution, order chosen via randomization.} \\
  Mila, McGill University \\ 
  \texttt{joey.bose@mail.mcgill.ca} \\
  \And
  Gauthier Gidel$^*$ \\
  Mila, Université de Montréal \\
    \texttt{gauthier.gidel@umontreal.ca} \\
  \And
  Hugo Berard$^*$\\
  Mila, Université de Montréal \\
  Facebook AI Research
   \And
  Andre Cianflone \\
  Mila, McGill University\\
  \And
  Pascal Vincent\thanks{Canada CIFAR AI Chair} \\
  Mila, Université de Montréal\\
  Facebook AI Research
  \And
  Simon Lacoste-Julien$^{\dagger}$ \\
  Mila, Université de Montréal\\
  \And
  William L. Hamilton$^{\dagger}$\\
  Mila, McGill University \\
}
\definecolor{mistyrose}{rgb}{1, 0.92, 0.92}

\definecolor{grannysmithapple}{rgb}{0.79, 0.92, 0.77}
\definecolor{darkgrannysmithapple}{rgb}{0.79, 0.99, 0.77}
\definecolor{palette1}{HTML}{8FC771}
\definecolor{palette2}{HTML}{B4D375}
\definecolor{palette3}{HTML}{DADF79}
\definecolor{palette4}{HTML}{FFEB7D}

\definecolor{medium-blue}{rgb}{0,0,1}
\begin{document}

\maketitle

\begin{abstract}
The existence of adversarial examples capable of fooling trained neural network classifiers calls for a much better understanding of possible attacks to guide the development of safeguards against them. This includes attack methods in the challenging {\em non-interactive blackbox} setting, where adversarial attacks are generated without any access, including queries, to the target model. Prior attacks in this setting have relied mainly on algorithmic innovations derived from empirical observations (e.g., that momentum helps), lacking principled transferability guarantees. In this work, we provide a theoretical foundation for crafting transferable adversarial examples to entire hypothesis classes. We introduce \textit{Adversarial Example Games} (AEG), a framework that models the crafting of adversarial examples as a min-max game between a generator of attacks and a classifier. AEG provides a new way to design adversarial examples by adversarially training a generator and a classifier from a given hypothesis class (e.g., architecture). We prove that this game has an equilibrium, and that the optimal generator is able to craft adversarial examples that can attack any classifier from the corresponding hypothesis class. We demonstrate the efficacy of AEG on the MNIST and CIFAR-10 datasets, outperforming prior state-of-the-art approaches with an average relative improvement of $29.9\%$ and $47.2\%$ against undefended and robust models (Table \ref{table:q2} \& \ref{table:q3}) respectively.

\end{abstract}

\section{Introduction}

Adversarial attacks on deep neural nets expose critical vulnerabilities in traditional machine learning systems \citep{moosavi2016deepfool,athalye2018synthesizing,sun2018survey,bose2018adversarial}. 
In order to develop models that are robust to such attacks, it is imperative that we improve our theoretical understanding of different attack strategies. 
While there has been considerable progress in understanding the theoretical underpinnings of adversarial attacks in relatively permissive settings (e.g. {\em whitebox} adversaries; \cite{madry2017towards}), there remains a substantial gap between theory and practice in more demanding and realistic threat models.

In this work, we provide a theoretical framework for understanding and analyzing adversarial attacks in the highly-challenging {\em \textbf{No}n-interactive black\textbf{Box} adversary} (NoBox) setting,
where the attacker has no direct access, including input-output queries, to the target classifier it seeks to fool.
Instead, the attacker must generate attacks by optimizing against some representative classifiers, which are assumed to come from a similar hypothesis class as the target.

The NoBox setting is a much more challenging setting than more traditional threat models, yet it is representative of many real-world attack scenarios, where the attacker cannot interact with the target model \citep{carlini2019evaluating}. 
Indeed, this setting---as well as the general notion of transferring attacks between classifiers---has generated an increasing amount of empirical interest \citep{dong2019evading,liu2016delving,xie2019improving,wu2020skip}.
The field, however, currently lacks the necessary theoretical foundations to understand the feasibility of such attacks.  

\xhdr{Contributions}
To address this theoretical gap, we cast NoBox attacks as a kind of {\em adversarial example game} (AEG).
In this game, an attacker generates adversarial examples to fool a representative classifier from a given hypothesis class, while the classifier itself is trained to detect the correct labels from the adversarially generated examples. Our first main result shows that the Nash equilibrium of an AEG leads to a distribution of adversarial examples effective against {\em any} classifier from the given function class.
More formally, this adversarial distribution is guaranteed to be the most effective distribution for attacking the hardest-to-fool classifiers within the hypothesis class, providing a worst-case guarantee for attack success against an arbitrary target. 
We further show that this optimal adversarial distribution admits a natural interpretation as being the distribution that maximizes a form of restricted conditional entropy over the target dataset, and we provide detailed analysis on simple parametric models to illustrate the characteristics of this optimal adversarial distribution. Note that while AEGs are latent games~\citep{gidel2020minimax}, they are distinct from the popular generative adversarial networks (GANs)~\citep{goodfellow2014generative}. In AEGs, there is \emph{no} discrimination task between two datasets (generated one and real one); instead, there is a standard supervised (multi-class) classification task on an adversarial dataset.

Guided by our theoretical results we instantiate AEGs using parametric functions ---i.e. neural networks, for both the attack generator and representative classifier and show the game dynamics progressively lead to a stronger attacker and robust classifier pairs. We empirically validate AEG on standard CIFAR and MNIST benchmarks and achieve state-of-the-art performance ---compared to existing heuristic approaches--- in nearly all experimental settings (e.g., transferring attacks to unseen architectures and attacking robustified models), while also maintaining a firm theoretical grounding.

\cut{
Empirical results on standard CIFAR and MNIST benchmarks using models that directly instantiate the AEG framework show that we can achieve state-of-the-art performance compared to existing heuristic approaches.
In particular, we find that our framework achieves state-of-the-art performance in nearly all experimental settings (e.g., transferring attacks to unseen architectures and attacking robustified models in a blackbox setting), while also maintaining a firm theoretical grounding. }

\section{Background and Preliminaries}

Suppose we are given a classifier $f : \mathcal{X}\rightarrow \mathcal{Y}$, an input datapoint $x \in \mathcal{X}$, and a class label $y \in \mathcal{Y}$, where $f(x) = y$.
The goal of an adversarial attack is to produce an adversarial example $x' \in \mathcal{X}$, such that $f(x') \neq y$, and where the distance\footnote{We assume that the $\ell_\infty$ is used in this work, \cite{goodfellow2014explaining, madry2017towards} , but our results generalize to any distance $d$.} $d(x,x') \leq \epsilon$. 
Intuitively, the attacker seeks to fool the classifier $f$ into making the wrong prediction on a point $x'$, which is $\epsilon$-close to a real data example $x$. 

\xhdr{Adversarial attacks and optimality}
A popular setting in previous research is to focus on generating {\em optimal} attacks on a single classifier $f$ \citep{carlini2017magnet, madry2017towards}.
Given a loss function $\ell$, used to evaluate $f$, an adversarial attack is said to be optimal if, 
\begin{equation}\label{eq:opt_attack}
    \textstyle x' \in \argmax_{x'\in \gX}\ell(f(x'),y) \,, \quad \text{s.t.} \quad d(x,x') \leq \epsilon \,.
\end{equation}
In practice, attack strategies that aim to realize ~\eqref{eq:opt_attack} optimize adversarial examples $x'$ directly using the gradient of $f$. 
In this work, however, we consider the more general setting of generating attacks that are optimal against entire hypothesis classes $\mathcal{F}$, a notion that we formalize below. 

\subsection{NoBox Attacks}
Threat models specify the formal assumptions of an attack (e.g., the information the attacker is assumed to have access to), which is a core aspect of adversarial attacks.
For example, in the popular {\em whitebox} threat model, the attacker is assumed to have full access to the model $f$'s parameters and outputs
\citep{szegedy2013intriguing,goodfellow2014explaining,madry2017towards}.
In contrast, the {\em blackbox} threat model assumes restricted access to the model, e.g., only access to a limited number of input-out queries \citep{chen2017zoo,ilyas2017query, papernot2016transferability}.
Overall, while they consider different access to the target model, traditional whitebox and blackbox attacks both attempt to generate adversarial examples that are optimal for a specific target (i.e., Equation \ref{eq:opt_attack}). 

In this paper, we consider the more challenging setting of {\em \textbf{no}n-interactive black\textbf{Box} (\textbf{NoBox})} attacks, intending to generate successful attacks against an unknown target.
In the NoBox setting, we assume no interactive access to a target model; instead, we only assume access to a target dataset and knowledge of the function class to which a target model belongs. Specifically, the NoBox threat model relies on the following key definitions:
\begin{itemize}[leftmargin=*, itemsep=2pt, topsep=2pt, parsep=2pt]
\item 
\textbf{The target model $f_t$}. The adversarial goal is to attack some target model $f_t : \mathcal{X} \rightarrow \mathcal{Y}$, which belongs to an hypothesis class $\gF$. Critically, the adversary has \emph{no access} to $f_t$ \emph{at any time}.
Thus, in order to attack $f_t$, the adversary must develop attacks that are effective against the entirety of $\gF$. 
\item
\textbf{The target examples $\mathcal{D}$}. The dataset $\mathcal{D}$ contains the examples $(x,y)$ that attacker seeks to corrupt. 
\item
\textbf{An hypothesis class $\gF$}. As noted above, we  assume that the attacker has access to a hypothesis class $\gF$ to which the target model $f_t$ belongs.\footnote{Previous work~\citep{tramer2017ensemble} usually assumes to have access to the architecture of $f_t$; we are more general by assuming access to a hypothesis class $\gF$ containing $f_t$; e.g., DenseNets can represent ConvNets.}  
One can incorporate in $\gF$ as much prior knowledge one has on $f_t$ (e.g., the architecture, dataset, training method, or regularization), going from exact knowledge of the target $\gF = \{f_t\}$ to almost no knowledge at all (e.g., $\gF = \{f \in \text{DenseNets}\}$). 
\item \textbf{A reference dataset $\gD_\text{ref}$}. The reference dataset $\gD_\text{ref}$, which is {\em similar} to the training data of the target model (e.g., sampled from the same distribution) is used to reduce the size of the hypothesis class $\gF$ (e.g., we know that the target model perfoms well at classification on $\gD_\text{ref}$).

\item
\textbf{A representative classifier $f_c$}. Finally, we assume that the attacker has the ability to optimize a representative classifier $f_c$ from the hypothesis class $\gF$. 
\end{itemize}

Given these four key components, we formalize the NoBox setting as follows:
\begin{definition}\label{def:nobox}
The NoBox threat model corresponds to the setting where the attacker (i) knows a hypothesis class $\mathcal{F}$ that the target model $f_t$ belongs to, (ii) has access to a reference dataset $\gD_\text{ref}$ that is similar to the the dataset used to train $f_t$ (e.g., sampled from the same distribution), and (iii)  can optimize a representative classifier $f_c \in \mathcal{F}$. 
The attacker has no other knowledge of---or access to---the target model $f_t$ (e.g., no queries to $f_t$ are allowed). The goal is, for the attacker, to use this limited knowledge to corrupt the examples in a given target dataset $\gD$.
\end{definition}

Our definition of a NoBox adversary (Def.~\ref{def:nobox}) formalizes similar notions used in previous work (e.g., see Def.~3 in~\cite{tramer2017ensemble}).
Previous work also often refers to related settings as generating {\em blackbox transfer} attacks, since the goal is to attack the target model $f_t$ while only having access to a representative classifier $f_c$ \citep{dong2019evading,liu2016delving, xie2019improving}.

Note, that our assumptions regarding dataset access are relatively weak. 
Like prior work, the attacker is given the target data (i.e., the examples to corrupt) as input, but this is constitutive of the task (i.e., we need access to a target example in order to corrupt it). 
Our only assumption is to have access to a reference dataset $\gD_\text{ref}$, which is {\em similar} to the dataset used to train the target model. We do not assume access to the exact training set. 
A stronger version of this assumption is made in prior works on blackbox transfer, as these approaches must craft their attacks on a known source model which is pretrained on the same dataset as the target model \citep{tramer2017ensemble}.

\section{Adversarial Example Games}
\label{sec:adv_example_game}

In order to understand the theoretical feasibility of NoBox attacks, we  view the attack generation task as a form of {\em adversarial game}.
The players are the {\em generator} network $g$---which learns a conditional distribution over adversarial examples---and the representative classifier $f_c$.
The goal of the generator network is to learn a conditional distribution of adversarial examples, which can fool the representative classifier $f_c$.
The representative classifier $f_c$, on the other hand, is optimized to detect the true label $y$ from the adversarial examples $(x', y)$ generated by $g$. 
A critical insight in this framework is that the generator and the representative classifier are \emph{jointly} optimized in a maximin game, making the generator's adversarial distribution at the equilibrium theoretically effective against {\em any} classifier from the hypothesis class $\mathcal{F}$ that $f_c$ is optimized over. 
At the same time, we will see in Proposition~\ref{prop:minimax_maximin} that the min and max in our formulation~\eqref{eq:two_player_game} can be switched. It implies that, while optimized, the model $f_c$ converges to a {\em robust classifier} against any attack generated by the generator~$g$~\citep{madry2017towards,wald1945statistical}, leading to increasingly powerful attacks as the adversarial game progresses. 

\xhdr{Framework} Given an input-output pair of target datapoints $(x,y)\sim \gD$, the generator network $g$ is trained to learn a distribution of adversarial examples $p_\text{cond}(\cdot|x,y)$ that---conditioned on an example to attack $(x,y)$---maps a prior distribution $p_{z}$ on $\gZ$ onto a distribution on $\gX$.  
The classifier network $f_c$  is simultaneously optimized to perform robust classification over the resulting distribution $p_{g}$ defined in ~\eqref{eq:def_pg} (below). 
Overall, the generator $g$ and the classifier $f_c$ play the following, two-player zero-sum game: 
\begin{equation}\label{eq:two_player_game} \tag{AEG}
     \max_{g \in \mathcal{G}_\epsilon}\min_{f_c \in \mathcal{F}}\mathbb{E}_{(x,y) \sim \gD, z\sim p_{z}}[ \ell(f_c(g(x,y,z)),y)] =: \varphi(f_c,g),
\end{equation}
where the generator $g \in \mathcal{G}_\epsilon$ is restricted by the similarity constraint $d(g(x,y,z),x) \leq \epsilon \,,\,\forall x,y,z \in \gX\times \gY\times\gZ$.
Once the generator $g$ is trained, one can generate adversarial examples against any classifier in $f_t \in \gF$, without queries, by simply sampling $z \sim p_z$ and computing $g(x,y,z)$. 

\xhdr{Connection with NoBox attacks}
The NoBox threat model (Def.~\ref{def:nobox}) corresponds to a setting where the attacker does not know the target model $f_t$ but only a hypothesis class $\gF$ such that $f_t \in \gF$. With such knowledge, one cannot hope to be better than the \emph{most pessimistic situation} where $f_t$ is the best defender in $\gF$. Our maximin formulation~\eqref{eq:two_player_game} encapsulates such a worst-case scenario, where the generator aims at finding attacks against the best performing $f$ in $\gF$.

\xhdr{Objective of the generator} When trying to attack infinite capacity classifiers---i.e., $\mathcal{F}$ contains any measurable function---the goal of the generator can be seen as generating the adversarial distribution $p_{g}$ with the highest expected conditional entropy $\E_x[ \sum_y p_{g}(y|x) \log p_{g}(y|x)]$, where $p_{g}$ is defined as
\vspace{-2mm}
\begin{equation}\label{eq:def_pg}
    (x',y) \sim p_{g} \Leftrightarrow x'= g(x,y,z)\,,\;(x,y) \sim \gD  \,,\,z\sim p_z
    \quad \text{with} \quad
    d(x',x) \leq \epsilon \, .
\end{equation}
When trying to attack a specific hypothesis class $\mathcal{F}$ (e.g., a particular CNN architecture), the generator aims at maximizing a notion of restricted entropy defined implicitly through the class $\mathcal{F}$.
Thus, the optimal generator in an~\eqref{eq:two_player_game} is primarily determined by the statistics of the target dataset $\mathcal{D}$ itself, rather any specifics of a target model. 
We formalize these high level concepts in \S\ref{sub:non_param}.

\xhdr{Regularizing the Game} In practice, the target $f_t$ is usually trained on a non-adversarial dataset and performs well at a standard classification task. 
In order to reduce the size of the class $\mathcal{F}$, one can bias the representative classifier $f_c$ towards performing well on a standard classification task with respect to $\gD_\textrm{ref}$, which leads to the following game:
\begin{equation}\label{eq:two_player_game_plus_clean}
         \max_{g \in \mathcal{G}_\epsilon}\min_{f_c \in \mathcal{F}}\mathbb{E}_{(x,y) \sim \gD, z\sim p_{z}}[ \ell(f_c(g(x,y,z)),y)] + \lambda \mathbb{E}_{(x,y) \sim \gD_\textrm{ref}}[\ell(f_c(x),y)] =: \varphi_\lambda(f,g).
\end{equation}
Note that $\lambda=0$ recovers~\eqref{eq:two_player_game}.
Such modifications in the maximin objective as well as setting the way the models are trained (e.g., optimizer, regularization, additional dataset) biases the training of the $f_c$ and corresponds to an implicit incorporation of prior knowledge on the target $f_t$ in the hypothesis class $\gF$. We note that in practice, using a non-zero value for $\lambda$ is essential to achieve the most effective attacks as the prior knowledge acts as a regularizer that incentivizes $g$ to craft attacks against classifiers that behave well on data similar to  $\gD_\textrm{ref}$.
\section{Theoretical results}
When playing an adversarial example game, the generator and the representative classifier try to beat each other by maximizing their own objective. In games, a standard notion of optimality is the concept of Nash equilibrium~\citep{nash1951non} where each player cannot improve its objective value by unilaterally changing its strategy. The minimax result in Prop.~\ref{prop:minimax_maximin} implies the existence of a Nash equilibrium for the game, consequently providing a well defined target for learning (we want to learn the equilibrium of that game). 
Moreover, a Nash equilibrium is a stationary point for gradient descent-ascent dynamics; we can thus hope for achieving such a solution by using a gradient-descent-ascent-based learning algorithm on~\eqref{eq:two_player_game}.\footnote{Note that, similarly as in practical GANs training, when the classifier and the generator are parametrized by neural networks, providing convergence guarantees for a gradient based method in such a nonconvex-nonconcave minimax game is an open question that is outside of the scope of this work.}

\begin{proprep}\label{prop:minimax_maximin} If $\ell$ is convex (e.g., cross entropy or mean squared loss), the distance $x \mapsto d(x,x')$ is convex for any $x' \in \gX$, one has access to any measurable $g$ respecting the proximity constraint in~\eqref{eq:def_pg}, and the hypothesis class $\mathcal{F}$ is convex, then we can switch min and max in~\eqref{eq:two_player_game}, i.e.,
\begin{equation}\label{eq:minimax}
        \min_{f_c \in \mathcal{F}} \max_{g \in \mathcal{G}_\epsilon}\varphi_\lambda(f_c,g)
        = 
       \max_{g \in \mathcal{G}_\epsilon}  \min_{f_c \in \mathcal{F}}\varphi_\lambda(f_c,g)
\end{equation}
\end{proprep}
\begin{proofsketch}We first notice that, by~\eqref{eq:def_pg} any $g$ corresponds to a distribution $p_g$ and thus we have,
\begin{equation}
    \varphi(f_c,g) := \mathbb{E}_{(x,y) \sim \gD, z\sim p_{z}}[ \ell(f_c(g(x,y,z)),y)] = \mathbb{E}_{(x',y) \sim p_g}[ \ell(f_c(x'),y)] =: \varphi(f_c,p_g) \notag
\end{equation}
Consequently, we also have $\varphi_\lambda(f_c,g) = \varphi_\lambda(f_c,p_g)$. By noting $\Delta_\epsilon := \{p_g \,:\, g \in \mathcal G_\epsilon\}$, we have that,
\begin{equation}
     \min_{f_c \in \mathcal{F}} \max_{p_g \in \Delta_\epsilon}\varphi_\lambda(f_c,p_g)
     =  \min_{f_c \in \mathcal{F}} \max_{g \in \mathcal{G}_\epsilon}\varphi_\lambda(f_c,g)
      \;\; \text{and} \;\;
       \max_{p_g \in \Delta_\epsilon} \min_{f_c \in \mathcal{F}}\varphi_\lambda(f_c,p_g)
     = \max_{g \in \mathcal{G}_\epsilon} \min_{f_c \in \mathcal{F}} \varphi_\lambda(f_c,g)
        \notag
\end{equation}
In other words, we can replace the optimization over the generator $g \in \mathcal{G}_\epsilon$ with an optimization over the set of possible adversarial distributions $\Delta_\epsilon$ induced by any $g \in \mathcal{G}_\epsilon$.
This equivalence holds by the construction of $\Delta_\epsilon$, which ensures that
$\max_{g \in \mathcal{G}_\epsilon}\varphi_\lambda(f_c, g) = \max_{p_g \in \Delta_\epsilon}\varphi_\lambda(f_c, p_g)$ for any $f_c \in \sF$.

We finally use Fan's theorem~\citep{fan1953minimax} after showing that $(f_c,p_g) \mapsto \varphi_\lambda(f_c,p_g)$ is convex-concave (by convexity of $\ell$ and linearity of $p_g \mapsto \E_{p_g}$) and that $\Delta_\epsilon$ is a compact convex set.
In particular, $\Delta_\epsilon$ is compact convex under the assumption that we can achieve any measurable $g$ (detailed in \S\ref{appendix}). 
\end{proofsketch}
\begin{appendixproof}

Let us recall that the payoff $\varphi$ is defined as 
\begin{equation}
     \max_{g \in \mathcal{G}_\epsilon}\min_{f \in \mathcal{F}}\mathbb{E}_{(x,y) \sim p_{data}, z\sim p_{z}}[ \ell(f(g(x,y,z)),y)] =: \varphi(f,g)
\end{equation}

Let us consider the case where $\gX$ is finite as a warm-up. In practice, this can be the case if we consider that for instance one only allows a finite number of values for the pixels, e.g. (integers between 0 and 255 for CIFAR-10). In that case, we have that 
\begin{equation}\label{eq:sum}
     \sP_{adv}(x,y) = \sum_{x' \in \gX} \sP_{data}(x',y) \sP_g(x|x',y) 
     \quad \text{where} \quad d(x,x') > \epsilon \Rightarrow \sP_g(x|x',y)= 0  \,.
\end{equation}
Assuming that one can achieve any $\sP_g(\cdot|x',y)$ respecting the proximity constraint, the set $\{\sP_{adv}\}$ is convex and compact. It is compact because closed and bounded in finite dimension and convex because of the linear dependence in $\sP_g$ in~\eqref{eq:sum} (and the fact that if $\sP_{g_1}$ and $\sP_{g_2}$ respect the constraints then $\lambda \sP_{g_1} + (1-\lambda) \sP_{g_2}$ does it too).

For the non finite input case, we can consider that the generator is a random variable defined on the probability space $(\gX\times \gY\times \gZ, \gB, \sP_{(x,y)} \times \sP_z := \sP)$
where $\gB$ is the Borel $\sigma$-algebra, $\sP_{(x,y)}$ the probability on the space of data and $\sP_z$ the probability on the latent space. Then the adversarial distributions $p_g$ we consider is the pushforward distributions $\sP \circ G^{-1}$ with $G$ such that, 
\begin{equation}\label{eq:constraints}
    G(x,y,z) = (g(x,y,z),y) 
    \quad \text{and} \quad
    d(g(x,y,z),x) \leq \epsilon \,,\, \forall x,y,z \in \gX\times \gY\times \gZ \,.
\end{equation}
It implies that $\Delta_\epsilon := \{p_g \,:\, g \in \mathcal G_\epsilon\} = \{\sP \circ G^{-1} \,|\, G \text{ measurable satisfying}~\eqref{eq:constraints}\}$. By definition of the payoff $\varphi$ and~\eqref{eq:def_pg}, we have that 
\begin{equation}
    \varphi(f_c,g) := \mathbb{E}_{(x,y) \sim \gD, z\sim p_{z}}[ \ell(f_c(g(x,y,z)),y)] = \mathbb{E}_{(x',y) \sim p_g}[ \ell(f_c(x'),y)] =: \varphi(f_c,p_g) \notag 
\end{equation}
and thus it lead to the equivalence, 
\begin{equation}
     \min_{f_c \in \mathcal{F}} \max_{p_g \in \Delta_\epsilon}\varphi_\lambda(f_c,p_g)
        = 
       \max_{p_g \in \Delta_\epsilon}  \min_{f_c \in \mathcal{F}}\varphi_\lambda(f_c,p_g)
       \Longleftrightarrow
       \min_{f_c \in \mathcal{F}} \max_{g \in \mathcal{G}_\epsilon}\varphi_\lambda(f_c,g)
        = 
       \max_{g \in \mathcal{G}_\epsilon}  \min_{f_c \in \mathcal{F}}\varphi_\lambda(f_c,g) \notag
\end{equation}

Recall that we assumed that one has access to any measurable $g$ that satisfies~\eqref{eq:constraints}, the minimax problem~\eqref{eq:minimax}. Let us show that under this assumption the set $\Delta_\epsilon$ is convex and compact.

$\Delta_\epsilon$ is convex: let us consider $\sP \circ G_1^{-1}$ and $\sP \circ G_2^{-1}$ we have that 
\begin{equation}
    \lambda\sP \circ G_1^{-1} + (1-\lambda)\sP \circ G_2^{-1}  = \sP \circ G_3^{-1}
\end{equation}
where $g_3(x,y,z) = \delta(z) g_1(x,y,z) + (1-\delta(z)) g_2(x,y,z)$ and where $\delta \sim Ber(\lambda)$. Note that $g_3$ satisfies~\eqref{eq:constraints} by convexity of $x \mapsto d(x,x')$.

$\Delta_\epsilon$ is compact: By using Skorokhod's representation theorem~\citep{billingsley1999convergence} we can show that this set is closed and thus compact (as closed subsets of compact sets are compact) using the weak convergence of measures as topology.

Thus $\Delta_\epsilon$ is a convex compact Haussdorf space and in both cases ($\gX$ finite and infinite) and we can apply Fan's Theorem.
\begin{theorem}\citep[Theorem 2]{fan1953minimax}
Let $U$ be a compact and convex Hausdorff space and $V$ an arbitrary convex set. Let $\varphi$ be a real valued function on $U \times V$ such that for every $v\in V$ the function $\varphi(\cdot,v)$ is lower semi-continuous on $U$.  If $\varphi$ is convex-concave then,
\begin{equation}
    \min_{u \in U} \sup_{v \in V} \varphi(u,v) = \sup_{v \in V} \min_{u \in U} \varphi(u,v)
\end{equation}
\end{theorem}

Note that we do not prove this result for neural networks.
\end{appendixproof}

The convexity assumption on the hypothesis class $\mathcal{F}$, Prop.~\ref{prop:minimax_maximin} applies in two main cases of interest:
\begin{enumerate*}[itemjoin = \;\,, label=(\roman*)]
    \item infinite capacity, i.e., when $\gF$ is any measurable function. 
    \item linear classifiers with \emph{fixed} features $\psi: \gX \to \sR^p$, i.e., $\gF = \{ w^\top \psi(\cdot) \,,\, w \in \sR^{|\gY| \times p} \}$. This second setting is particularly useful to build intuitions on the properties of~\eqref{eq:two_player_game}, as we will see in \S\ref{sub:logistic_reg} and Fig.~\ref{fig:self_info}.
\end{enumerate*}
The assumption that we have access to any measurable $g$, while relatively strong, is standard in the literature and is often stated in prior works as ``if $g$ has enough capacity''~\citep[Prop. 2]{goodfellow2014explaining}. Even if the class of neural networks with a fixed architecture do not verify the assumption of this proposition, the key idea is that neural networks are good candidates to approximate that equilibrium because they are universal approximators~\citep{hornik1991approximation} and they form a set that is ``almost convex"~\citep{gidel2020minimax}. Proving a similar minimax theorem by only considering neural networks is a challenging problem that has been considered by~\citet{gidel2020minimax} in a related setting. It requires a fined grained analysis of the property of certain neural network architecture and is only valid for approximate minimax. We believe such considerations outside of the scope of this work.

\subsection{A simple setup: binary classification with logistic regression}
\label{sub:logistic_reg}
Let us now consider a binary classification setup where $\gY = \{\pm 1\}$ and $\gF$ is the class of linear classifiers with linear features, i.e $f_w(x) = w^\top x$. In this case, the payoff of the game~\eqref{eq:two_player_game} is,
\begin{equation}\label{eq:two_player_logistic}
    \varphi(f_\omega,g):= \mathbb{E}_{(x,y)\sim \gD,\, z\sim p_z}[ \log(1+e^{-y \cdot w^\top g(x,y,z)})]
\end{equation}
This example is similar to the one presented in~\citep{goodfellow2014explaining}. However, our purpose is different since we focus on characterizing the optimal generator in~\eqref{eq:minimax}.
We show that the optimal generator can attack any classifier in $\gF$ by shifting the means of the two classes of the dataset $\gD$. 
\begin{proprep}\label{prop:adv_reg} If the generator is allowed to generate any $\ell_\infty$ perturbations. The optimal linear representative classifier is the solution of the following $\ell_1$ regularized logistic regression
\begin{equation}\label{eq:minimzation_robust_log}
   w^* \in \arg \min_{w} \mathbb{E}_{(x,y)\sim \gD}[ \log(1+e^{-y \cdot w^\top x + \epsilon\|w\|_1})] \,.
\end{equation}
 Moreover if $\omega^*$ has no zero entry, the optimal generator is $g^*(x,y) = x - y \cdot \epsilon \sign(w^*)$, is \emph{deterministic} and the pair $(f_{w^*},g^*)$ is a Nash equilibrium of the game~\eqref{eq:two_player_logistic}.
\end{proprep}
\begin{appendixproof}
We prove the result here for any given norm $\|\cdot\|$.
Let us consider the loss for a given pair $(x,y)$
\begin{equation}
    \log(1+e^{y(w^\top g(x,y) + b)})
\end{equation}
then by the fact that $x\mapsto \log(1+e^x)$ is increasing, maximizing this term for $\|g(x,y) - x\|_\infty \leq  \epsilon$, boils down to solving the following maximization step,
\begin{equation}
    \max_{\delta\,,\, \|\delta\| \leq \epsilon } y(w^\top \delta) = \|w\|_*
\end{equation}
Particularly, for the $\ell_\infty$ norm we get 
\begin{equation}
    \arg \max_{\delta\,,\, \|\delta\|_\infty \leq \epsilon } y(w^\top \delta) = \epsilon y \sign(w) \,.
\end{equation}
and
\begin{equation}
   \max_g \mathbb{E}_{(x,y)\sim p_{data}}[ \log(1+e^{y(w^\top g(x,y) + b)})]
   = \mathbb{E}_{(x,y)\sim p_{data}}[ \log(1+e^{y(w^\top x + b) + \epsilon\|w\|_1 })]
\end{equation}

To show that $(f^*,g^*)$ is a Nash equilibrium of the game~\eqref{eq:two_player_logistic}, we first notice that by construction
\begin{equation}
    \varphi(f^*,g^*) = \min_{f \in \mathcal{F}} \max_g \varphi(f,g)
\end{equation}
where $\mathcal{F}$ is the class of classifier with linear logits. We then just need to notice that for all $f \in \mathcal{F}$ we have,
\begin{equation}\label{eq:conv_pb}
    \varphi(f,g^*) =  \mathbb{E}_{(x,y)\sim p_{data}}[ \log(1+e^{-y(w^\top x + b) + \epsilon w^\top \sign(w^*)})]
\end{equation}
That is a convex problem in $(w,b)$. Since, in a neighborhood of $w^*$ we have that $w^\top \sign(w^*) = \|\restr{w}{\supp(w^*)}\|_1$. Thus by assuming that $w^*$ is full support and since $w^*$ minimize~\eqref{eq:minimzation_robust_log} we have that 
\begin{equation}
    \nabla_w \varphi(w^*,b^*,g^*) = \nabla_{w} \mathbb{E}_{(x,y)\sim p_{data}}[ \log(1+e^{-y(w^\top x + b) + \epsilon\|w\|_1})] = 0
\end{equation}
Finally, by convexity of the problem~\eqref{eq:conv_pb} we can conclude that $w^*$ is a minimizer of $\varphi(\cdot,g^*)$. To sum-up we have that
\begin{equation}
    \varphi(f^*,g) \leq \varphi(f^*,g^*) \leq \varphi(f,g^*)
\end{equation}
meaning that $(f^*,g^*)$ is a Nash equilibrium of~\eqref{eq:two_player_logistic}.
\end{appendixproof}
A surprising fact is that, unlike in the general setting of Prop.~\ref{prop:minimax_maximin}, the generator in Prop.\ref{prop:adv_reg} is deterministic (i.e., does not depend on a latent variable $z$).\footnote{Note also that one can generalize Prop.~\ref{prop:adv_reg} to a perturbation with respect to a general norm $\|\cdot\|$, in that case, the $\epsilon$-regularization for the classifier would be with respect to the dual norm $\|\cdot \|_* := \max_{\|u\| \leq 1} \langle\cdot,u \rangle $. E.g., as previously noted by~\citet{goodfellow2014explaining}, $\ell_\infty$ adversarial perturbation leads to a $\ell_1$ regularization.} This follows from the simple structure of classifiers in this class, which allow for a closed form solution for $g^*$. In general, one cannot expect to achieve an equilibrium with a deterministic generator.
Indeed, with this example, our goal is simply to illustrate how the optimal generator can attack an entire class of functions with limited capacity: linear classifiers are mostly sensitive to the mean of the distribution of each class; the optimal generator exploits this fact by moving these means closer to the decision boundary.

\subsection{General multi-class classification}\label{sub:non_param}
In this section, we show that, for a given hypothesis class $\gF$, the generated distribution achieving the global maximin against $f_c \in \gF$ can be interpreted as the distribution with the highest $\mathcal{F}$-entropy. For a given distribution $p_g$, its $\mathcal{F}$-entropy is the minimum expected risk under $p_g$ one can achieve in $\gF$. 
\begin{definition}\label{def:F_entropy}
For a given distribution $(x,y) \sim p_{g}$ we define the $\mathcal{F}$-entropy of $p_g$ as
\begin{equation}
    H_{\mathcal{F}}(p_{g}) := \min_{f_c \in \mathcal{F}} \E_{(x,y) \sim p_g} [ \ell(f_c(x),y)]
    \qquad \text{where $\ell$ is the cross entropy loss.}
\end{equation}
\vspace{-5mm}
\end{definition}
Thus $\mathcal{F}$-entropy quantifies the amount of ``classification information" available in $p_g$ using the class of classifiers $\mathcal{F}$. If the $\gF$-entropy is large, $(x,y) \sim p_g$ cannot be easily classified with a function $f_c$ in $\gF$.
Moreover, it is an upper-bound on the \emph{expected conditional entropy} of the distribution $p_g$.
\begin{proprep} \label{prop:self_info}
The $\mathcal{F}$-entropy is a decreasing function of $\mathcal{F}$, i.e., for any $\mathcal{F}_1 \subset \mathcal{F}_2$,
\begin{equation}
    H_{\mathcal{F}_1}(p_g) \geq H_{\mathcal{F}_2}(p_g)  \geq H_y(p_g):= \E_{x\sim p_x}[H(p_g(\cdot|x))]\,.
    \notag
\end{equation}
where $H(p(\cdot|x)) := \sum_{y \in \gY} p(y|x) \ln p(y|x)$ is the entropy of the conditional distribution $p(y|x)$.
\end{proprep}
\begin{appendixproof}
In the case where $f^*(x) :=p(y|x) \in \mathcal{F}$, since $f^*$ a minimizer of the expected cross entropy loss over the class of any function, we have that 
\begin{equation}\label{eq:cond_expected_entropy}
    H_y(p) := \min_{f\in \gX^{\gY}} \E_{(x,y) \sim p} [ \ell(f(x),y)] = \E_x[H(p(\cdot|x))]
\end{equation}

\begin{lemma} Given a data distribution $(x,y) \sim p_{adv}$ a minimizer of the cross entropy loss is 
\begin{equation}
     p_{adv}(y|\cdot)  \in \arg \min_{f} \E_{(x,y)\sim p_{adv}}[\ell(f(x),y)] \,.
\end{equation}
\end{lemma}
\begin{proof}
Let us start by noticing that,
\begin{equation}
    \min_f \E_{(x,y)\sim p}[\ell(f(x),y)] 
    = \E_{x \sim p_x} \min_{q = f(x)} \E_{y \sim p(\cdot|x)} [\ell(q,y)]
\end{equation}
Using the fact that $\ell$ is the cross-entropy loss we get
\begin{equation}
    \E_{y \sim p(\cdot|x)} [\ell(q,y)] = \E_{y \sim p(\cdot|x)} [-\sum_{i=1}^K y_i \ln(q_i)] =   -\sum_{i=1}^K p_i \ln(q_i)
\end{equation}
where we noted $p_i = p(y=i|x)$. Noticing that since $q_i$ is a probability distribution we have $\sum_{i=1}^K q_i=1$, we have,
\begin{equation}
    \E_{y \sim p(\cdot|x)} [\ell(q,y)] =  -\sum_{i=1}^{K-1} p_i \ln(q_i) - p_K \ln(1-\sum_{i=1}^{K-1}q_i)
\end{equation}
we can then differentiate this loss with respect to $q_i \geq 0$ and get,
\begin{equation}
    \frac{\partial\E_{y \sim p(\cdot|x)} [\ell(q,y)]}{\partial q_i}(q) = - \frac{p_i}{q_i} + \frac{p_K}{q_K}
\end{equation}
We can finally notice that $q_i = p_i$ is a feasible solution.
\end{proof}
\end{appendixproof}
Here $p_g$ is defined as in \eqref{eq:def_pg} and implicity depends on $\mathcal{D}$. For a given class $\mathcal F$, the solution to an~\eqref{eq:two_player_game} game can be seen as one which finds a regularized adversarial distribution of maximal $\mathcal{F}$-entropy,
\begin{equation}
    \max_{g\in \mathcal{G}_\epsilon} \min_{f_c \in \mathcal{F}} \varphi_\lambda(f_c,g) = (1+\lambda)\max_{g\in \mathcal{G}_\epsilon} H_{\mathcal{F}}( \tfrac{1}{(1+\lambda)}p_{g} + \tfrac{\lambda}{(1+\lambda)}\gD_\text{ref})] \,,
\end{equation}
where the distribution $\tfrac{1}{(1+\lambda)}p_{g} + \tfrac{\lambda}{(1+\lambda)}\gD_\text{ref}$ is the mixture of the generated distribution $p_g$ and the empirical distribution over the dataset $\gD_\text{ref}$.
This alternative perspective on the game~\eqref{eq:two_player_game} shares similarities with the divergence minimization perspective on GANs~\citep{huang2017parametric}. 
However, while in GANs it represents a divergence between two distributions, in~\eqref{eq:two_player_game} this corresponds to a notion of entropy.

A high-level interpretation of $\mathcal{F}$-entropy maximization is that it implicitly defines a metric for distributions which are challenging to classify with only access to classifiers in $\mathcal{F}$. Overall, the optimal generated distribution $p_g$ can be seen as the most adversarial dataset against the class $\mathcal{F}$.

\xhdr{Properties of the $\mathcal{F}$-entropy} 
\begin{figure}
\centering
\vspace{-1cm}
\begin{subfigure}
    \centering
    \includegraphics[width=.67\linewidth]{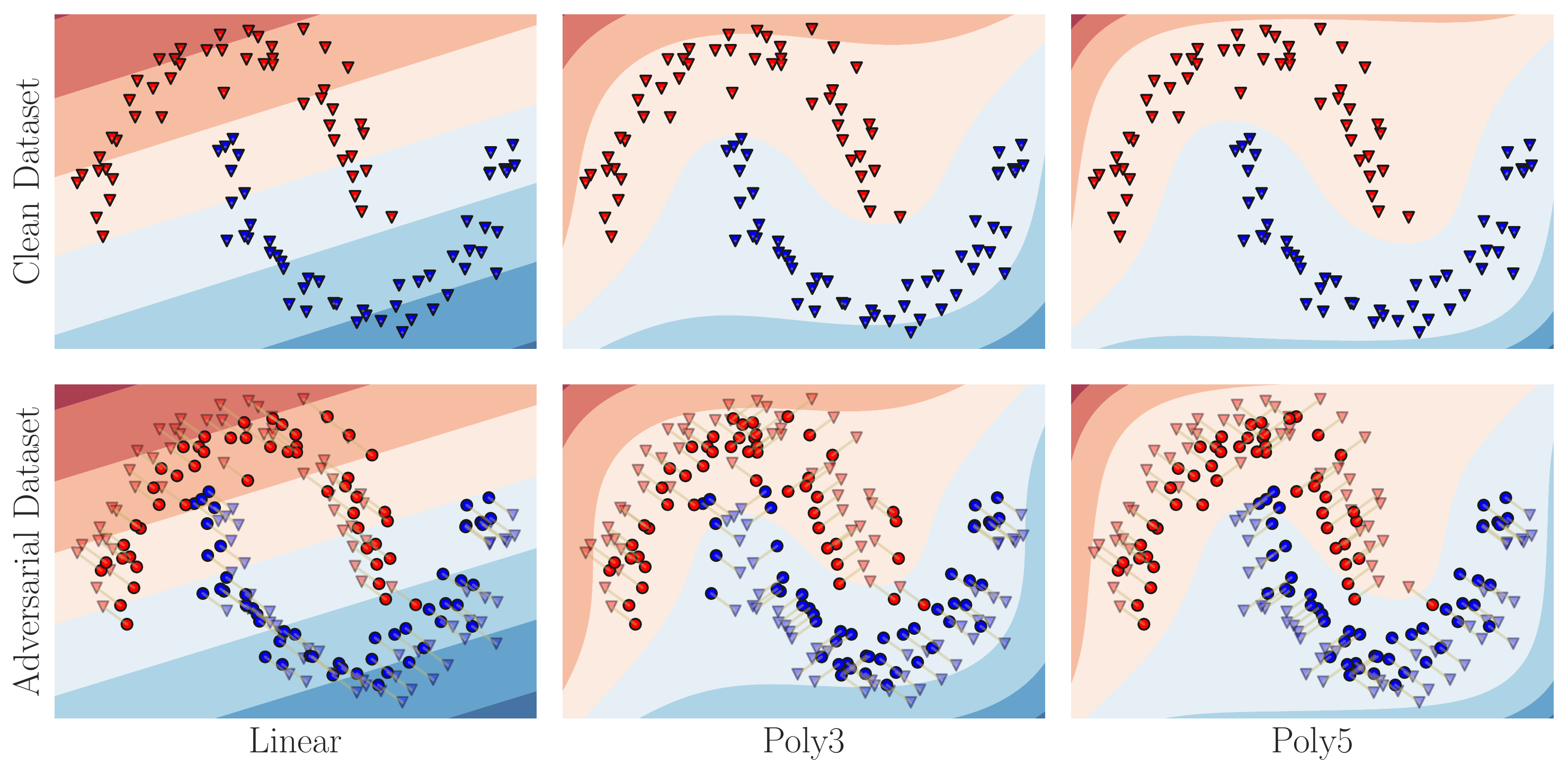}
    \vspace{-2mm}
\end{subfigure}
\hspace{-2mm}
\begin{subfigure}
\centering
\includegraphics[width=.32\linewidth]{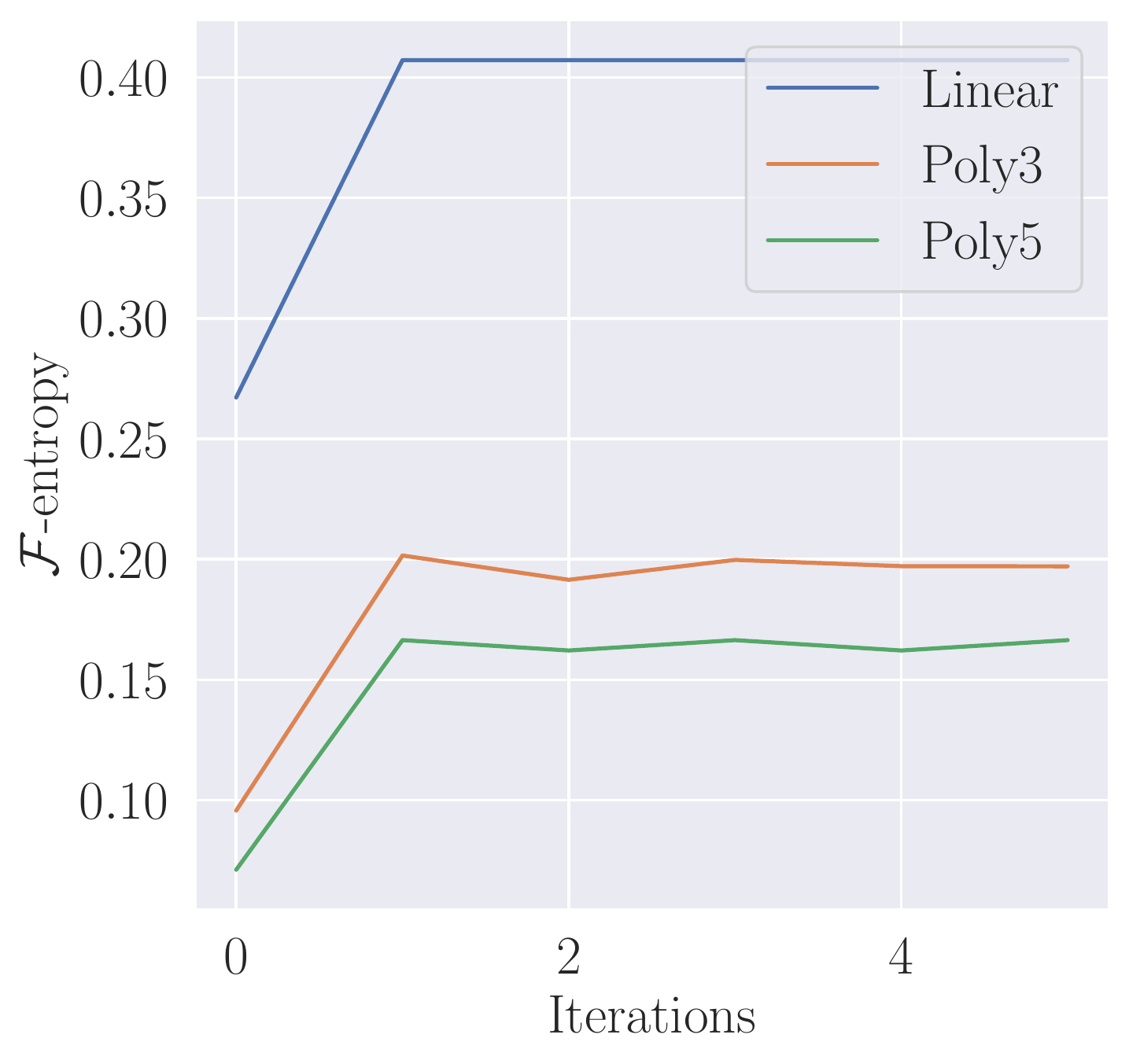}
\end{subfigure}
    \caption{ \small
    Illustration of Proposition~\ref{prop:self_info} for three classes of classifiers in the context of logistic regression for the two moon dataset of scikit-learn~\citep{scikit-learn} with linear and polynomial (of degree 3 and 5) features. \textbf{Left:} Scatter plot of the clean or adversarial dataset and the associated optimal decision boundary. For the adversarial dataset, each corresponding clean example is represented with a {\color{red} $\blacktriangle$}/{\color{blue} $\blacktriangle$} and is connected to its respective adversarial example {\color{red}$\bullet$}/{\color{blue}$\bullet$}. \textbf{Right:} value of the $\mathcal{F}$-entropy for the different classes as a function of the number of iterations.}
    \label{fig:self_info}
\end{figure}
We illustrate the idea that the optimal generator and the $\gF$-entropy depend on the hypothesis class $\gF$ using a simple example. To do so, we perform logistic regression~\eqref{eq:two_player_logistic} with linear and polynomial (of degree 3 and 5) features (respectively called Linear, Poly3, and Poly5) on the two moon dataset of scikit-learn~\citep{scikit-learn}. Note that we have $\text{Linear} \subset \text{Poly3} \subset \text{Poly5}$. For simplicity, we consider a deterministic generator $g(x,y)$ that is realized by computing the maximization step via 2D grid-search on the $\epsilon$ neighborhood of $x$. We train our models by successively fully solving the minimization step and the maximization step in~\eqref{eq:two_player_logistic}. 

We present the results in Figure~\ref{fig:self_info}. One iteration corresponds to the computation of the optimal classifier against the current adversarial distribution $p_g$ (also giving the value of the $\gF$-entropy), followed by the computation of the new optimal adversarial $p'_{g}$ against this new classifier. The left plot illustrates the fact that the way of attacking a dataset depends on the class considered. For instance, when considering linear classifiers, the attack is a uniform translation on all the data-points of the same class. While when considering polynomial features, the optimal adversarial dataset pushes the the corners of the two moons closer together. In the right plot, we can see an illustration of Proposition~\ref{prop:self_info}, where the $\mathcal{F}$-entropy takes on a smaller value for larger classes of classifiers.

\section{Attacking in the Wild: Experiments and Results}

We investigate the application of our AEG framework to produce adversarial examples against
MNIST and CIFAR-10 classifiers. First we investigate our performance in a challenging NoBox setting where we must attack an unseen target model with knowledge of only its hypothesis class (i.e., architecture) and a sample of similar training data (\S\ref{q1}). Following this, we investigate how well AEG attacks transfer across architectures (\S\ref{q2}), as well as AEG's performance attacking robust classifiers (\S\ref{q3}). 

\xhdr{Experimental setup}
We perform all attacks, including baselines, with respect to the $\ell_{\infty}$ norm constraint with $\epsilon=0.3$ for MNIST and $\epsilon=0.03125$ for CIFAR-10. For AEG models, we train both  generator ($g$) and representative classifier ($f_c$) using stochastic gradient descent-ascent with the ExtraAdam optimizer~\citep{gidel2019variational} and held out target models, $f_t$, are trained offline using SGD with Armijo line search \citep{vaswani2019painless}.
Full details of our model architectures, including hyperparameters, employed in our AEG framework can be found in Appendix \S{\ref{app:architecture}}.\footnote{Code: \url{https://github.com/joeybose/Adversarial-Example-Games.git}}

\xhdr{Baselines}
Throughout our experiments we rely on four standard blackbox transfert attack strategies adapted to the NoBox setting: the Momentum-Iterative Attack (MI-Attack) \citep{dong2018boosting}, the Input Diversity  (DI-Attack) \citep{xie2019improving}, the Translation-Invariant (TID-Attack) \citep{dong2019evading} and the Skip Gradient Method (SGM-Attack) \citep{wu2020skip}.
For fair comparison, we inherit all hyperparameter settings from their respective papers. Note that SGM-attack is only defined with architectures that contain skip connections (e.g. ResNets).

\xhdr{AEG Architecture} The high-level architecture of our AEG framework is illustrated in Figure~\ref{fig:AEG_architecture}. The generator takes the input $x$ and encode it into $\psi(x)$, then the generator uses this encoding to compute a probability vector $p(\psi(x))$ in the probability simplex of size $K$, the number of classes. Using this probability vector, the network then samples a categorical variable $z$ according to a multinomial distribution of parameter $p(\psi(x))$. Intuitively, this category may correspond to a target for the attack. The gradient is backprogated across this categorical variable using the gumble-softmax trick \cite{jang2016categorical,maddison2016concrete}. Finally, the decoder takes as input $\psi(x)$, $z$ and the label $y$ to output an adversarial perturbation $\delta$ such that $\|\delta\|\leq \epsilon$. In order to generate adversarial perturbations over images that obey $\epsilon$-ball constraints, we employ a scaled tanh output layer to scale the output of the generator to $(0,1)$, subtract the clean images, and finally apply an elementwise multiplication by $\epsilon$. We then compute $\ell(f(x+\delta),y)$ where $f$ is the critic and $\ell$ the cross entropy loss. Further details can be found in Appendix~\S\ref{app:architecture}.

\begin{figure}
    \centering
    \includegraphics[width= .9 \linewidth]{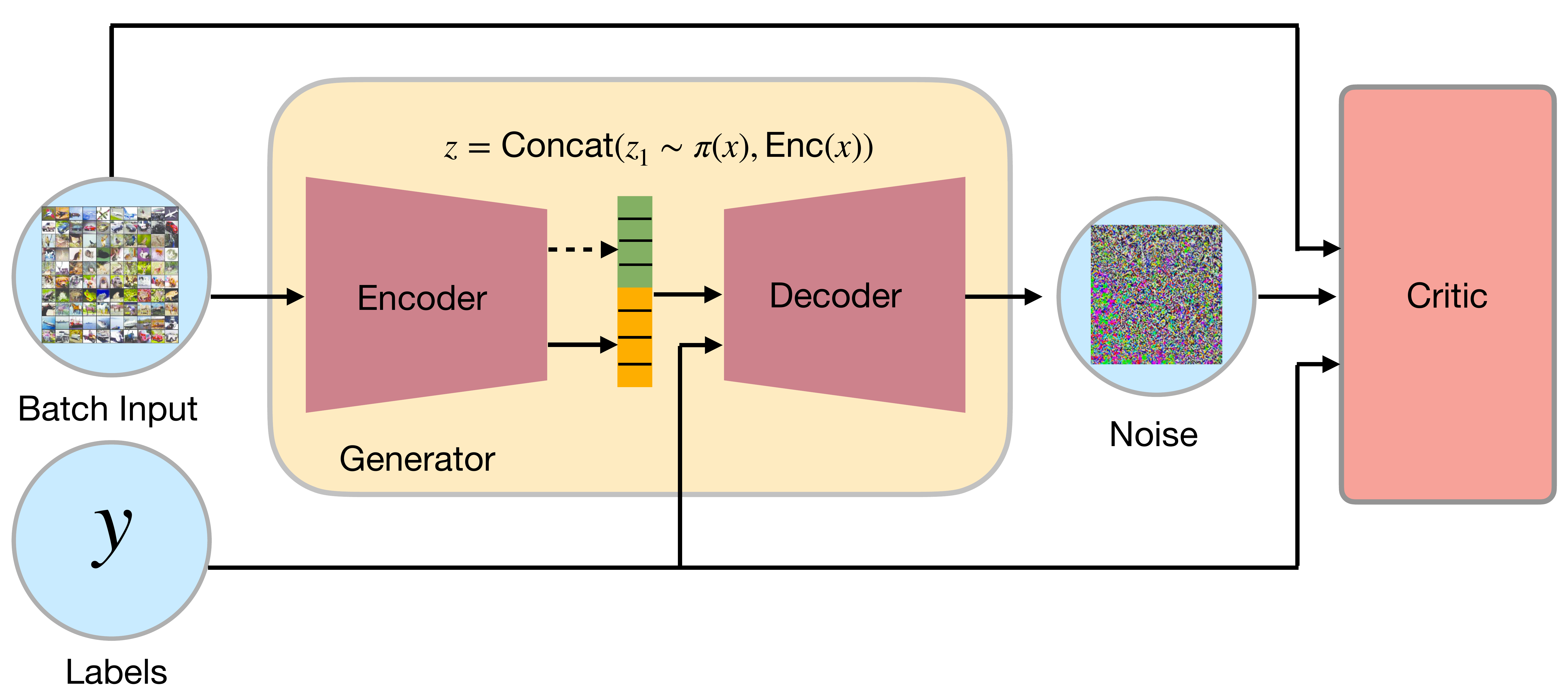}
    \caption{AEG framework architecture}
\label{fig:AEG_architecture}
\end{figure}
\cut{
\begin{wrapfigure}{r}{7.5cm}
\vspace{-4mm}
\centering
\includegraphics[width= \linewidth]{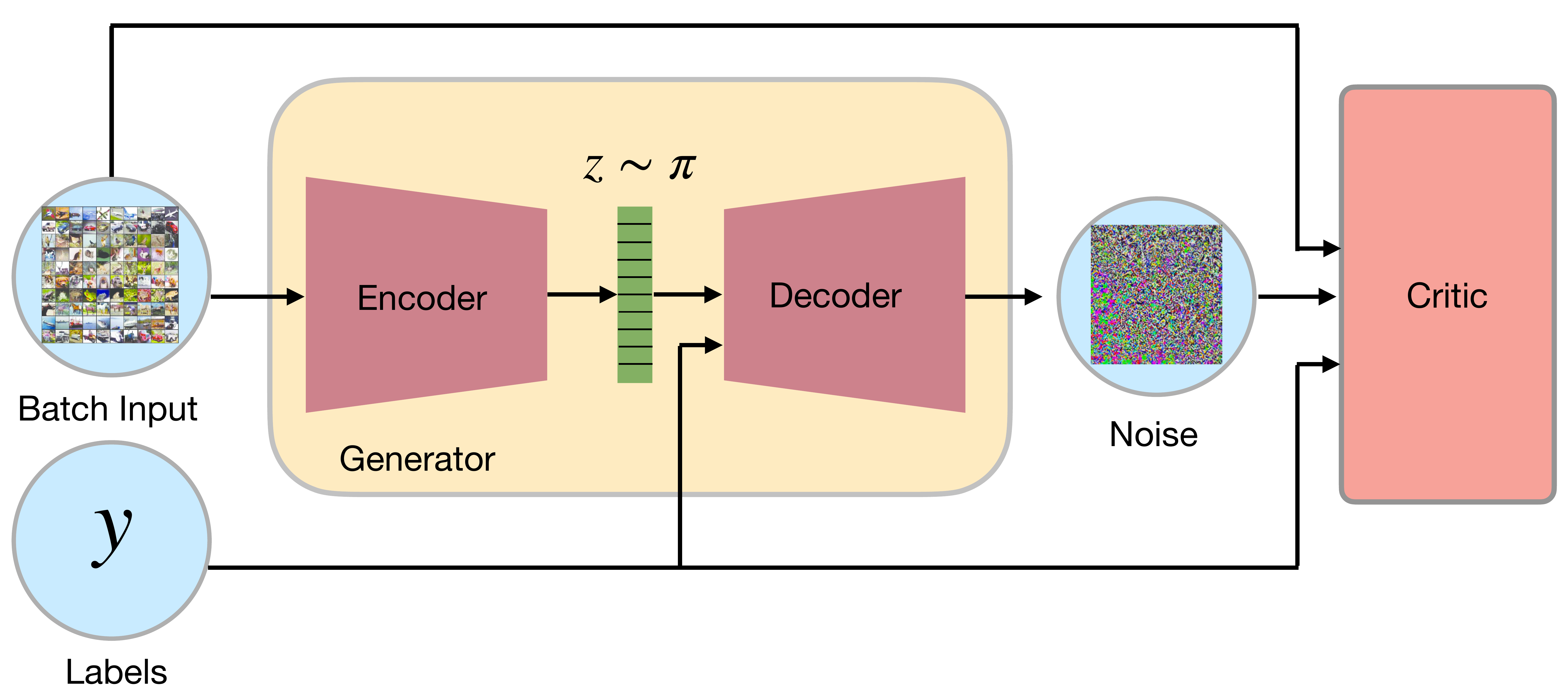}
\caption{AEG Architecture.}
\vspace{-2mm}
\end{wrapfigure}
}

\cut{
We allow the generator to update its parameters several times on the same batch of examples before updating the critic. In particular we update the generator until it is able to fool the critic or it reaches some fixed number of iterations. We set this max number of iterations to 20 in all our experiments. We also find that biasing the critic update various forms of adversarial training consistently leads to the most effective attack. We reconcile this phenomenon by noting that through adversarial training the critic itself becomes a robust model which provides a richer learning signal to the generator. Furthermore, the elegance of the AEG frameworks allows the practitioner to further bias the optimization process of the critic ---and consequently the generator--- through picking and choosing effective robustness techniques such as training with PGD adversarial examples generated at a prior timestep.
}
\subsection{NoBox Attacks on a Known Architecture Class but Unknown Train Set}\label{q1}

We first evaluate the AEG framework in a NoBox setting, where we know only the architecture of the target model and have only access to a sample of similar training data (but not the exact training data of the target model).
To simulate having access to a similar (but not identical dataset) as the target model, for each dataset we create random equally-sized splits of the data (10000 examples per splits). Within each split we use one fold to train the split classifier which acts as the representative classifiers for all attackers who are then evaluated their ability to fool the remaining split classifiers on unseen target examples $\mathcal{D}$.
For the MNIST dataset we consider LeNet classifier \citep{lecun2015lenet}, while for CIFAR-10 we consider ResNet-18 \citep{he2016deep}.
Table \ref{q1_table_new} shows the results of our experiments on this task, averaged across all splits and folds. 
We see that our AEG approach achieves state-of-the-art results, either outperforming or matching (within a 95\% confidence interval) all baselines in both settings. 
Note that this task is significantly more challenging than many prior blackbox attack setups, which assume access to the full training data of the target model.\footnote{We include results on a more permissive settings with access to the full training data in Appendix \ref{appendix:quant_results}}

\begin{table}[ht]
    \centering
 \label{q1_table_new}
    \begin{small}
    \centering
            \begin{tabular}{cccccc}
            \toprule
                   Dataset & MI-Attack & DI-Attack & TID-Attack & SGM-Attack &  AEG (Ours) \\
                    \midrule
                    MNIST  & \cellcolor{palette2} 87.5 $\pm$ 2.7 & \cellcolor{palette1}  \textbf{89.5 $\pm$ 2.5} & \cellcolor{palette3} 85.4 $\pm$ 2.8 $^\dag$& N/A & \cellcolor{palette1} \textbf{89.5 $\pm$ 3.2} \\
                    \midrule
                    CIFAR-10 (Res18) &\cellcolor{palette3}56.8 $\pm$ 1.2  $^\dag$& \cellcolor{palette2}  84.0 $\pm$ 1.5 $^\dag$& \cellcolor{palette4}  9.1 $\pm$ 1.6 $^\dag$& \cellcolor{palette3}  60.5 $\pm$ 1.5 $^\dag$& \cellcolor{palette1} \textbf{87.0 $\pm$ 2.1}\\
            \bottomrule
            \end{tabular}
    \end{small}
     \caption{Attack success rates, averaged across target models with 95\% confidence intervals shown. $^\dag$indicates a statistically significant result as determined by the paired T-test when compared to AEG.}
\end{table}

\subsection{NoBox Attacks Across Distinct Architectures}
\label{q2}
We now consider NoBox attacks where we do not know the architecture of the target model but where the training data is known---a setting previously referred to as blackbox transfer \citep{tramer2017ensemble}. For evaluation, we use CIFAR-10 and train $10$ instances of VGG-16 \citep{simonyan2014very}, ResNet-18 (RN-18) \citep{he2016deep}, Wide ResNet (WR) \citep{zagoruyko2016wide}, DenseNet-121 (DN-121) \citep{huang2017densely} and Inception-V3 architectures (Inc-V3) \citep{szegedy2016rethinking}. 
Here, we optimize the attack approaches against a single pre-trained classifier from a particular architecture and then evaluate their attack success on classifiers from distinct architectures averaged over $5$ instantiations.  
Our findings when using ResNet-18, DenseNet-121 and the VGG-16 as the source architecture are provided in Table 2. Overall we find that AEG beats all other approaches and lead to a new state of the art. In particular AEG outperforms the best baseline in each setting by an average of $29.9\%$ across the different source architectures with individual average gains of $9.4\%$, $36.2\%$, and $44.0\%$ when using a RN-18 model, DN-121, and VGG-16 source models respectively.

\begin{table}[ht]
\begin{small}
    \centering
            \begin{tabular}{ccccccc}
            \toprule
                   Source & Attack &   VGG-16 & RN-18 & WR & DN-121 & Inc-V3 \\
                            \midrule
                            & Clean &  11.2 $\pm$ 1.8 &  13.1 $\pm$ 4.0 &  6.8 $\pm$ 1.4 &  11.2 $\pm$ 2.8 &  9.9 $\pm$ 2.6 \\
                            \midrule
\multirow{5}{*}{RN-18}      & MI-Attack & \cellcolor{palette3}  63.9 $\pm$ 2.6 & \cellcolor{palette3}  74.6 $\pm$ 0.8 & \cellcolor{palette3}  63.1 $\pm$ 2.4 & \cellcolor{palette3}  72.5 $\pm$ 2.6 & \cellcolor{palette3}  67.9 $\pm$ 3.2 \\
                            & DI-Attack & \cellcolor{palette2}  77.4 $\pm$ 3.4 & \cellcolor{palette2}  90.2 $\pm$ 1.6 & \cellcolor{palette2}  74.0 $\pm$ 2.0 & \cellcolor{palette2}  87.1 $\pm$ 2.6 &\cellcolor{palette1} \textbf{85.8 $\pm$ 1.6} \\
                            & TID-Attack & \cellcolor{palette4}  21.6 $\pm$ 2.6 & \cellcolor{palette4}  26.5 $\pm$ 4.8 & \cellcolor{palette4}  14.0 $\pm$ 3.0 & \cellcolor{palette4}  22.3 $\pm$ 3.2 & \cellcolor{palette4}  19.8 $\pm$ 1.8 \\
                            & SGM-Attack & \cellcolor{palette3}  68.4 $\pm$ 3.6 & \cellcolor{palette3}  79.5 $\pm$ 1.0 & \cellcolor{palette3}  64.3 $\pm$ 3.2 & \cellcolor{palette3}  73.8 $\pm$ 2.0& \cellcolor{palette3}  70.6 $\pm$ 3.4 \\
                            & AEG (Ours) &\cellcolor{palette1} \textbf{93.8 $\pm$ 0.7} &\cellcolor{palette1} \textbf{97.1 $\pm$ 0.4}  &\cellcolor{palette1} \textbf{80.2 $\pm$ 2.2} &\cellcolor{palette1} \textbf{93.1 $\pm$ 1.3} & \cellcolor{palette1} \textbf{88.4 $\pm$ 1.6}\\
                            \midrule
\multirow{5}{*}{DN-121}     & MI-Attack & \cellcolor{palette3}  54.3 $\pm$ 2.2 & \cellcolor{palette3}  62.5 $\pm$ 1.8 & \cellcolor{palette3}  56.3 $\pm$ 2.6 & \cellcolor{palette3}  66.1 $\pm$ 3.0 & \cellcolor{palette3}  65.0 $\pm$ 2.6 \\
                            & DI-Attack & \cellcolor{palette3}  61.1 $\pm$ 3.8 & \cellcolor{palette3}  69.1 $\pm$ 1.6 & \cellcolor{palette3}  61.9 $\pm$ 2.2 & \cellcolor{palette2}  77.1 $\pm$ 2.4 & \cellcolor{palette3}  71.6 $\pm$ 3.2 \\
                            & TID-Attack & \cellcolor{palette4} 21.7 $\pm$ 2.4 & \cellcolor{palette4}  23.8 $\pm$ 3.0 & \cellcolor{palette4}  14.0 $\pm$ 2.8 & \cellcolor{palette4}  21.7 $\pm$ 2.2 & \cellcolor{palette4}  19.3 $\pm$ 2.4 \\
                            & SGM-Attack & \cellcolor{palette3} 51.6 $\pm$ 1.4 & \cellcolor{palette3}  60.2 $\pm$ 2.6 & \cellcolor{palette3}  52.6 $\pm$ 1.8 & \cellcolor{palette3}  64.7 $\pm$ 3.2 & \cellcolor{palette3} 61.4 $\pm$ 2.6  \\
                            & AEG (Ours) &\cellcolor{palette1} \textbf{93.7 $\pm$ 1.0} &\cellcolor{palette1} \textbf{97.3 $\pm$ 0.6}  &\cellcolor{palette1} \textbf{81.8 $\pm$ 3.0} &\cellcolor{palette1} \textbf{96.7 $\pm$ 0.8} &\cellcolor{palette1} \textbf{92.7 $\pm$ 1.6} \\
                           \midrule
\multirow{5}{*}{VGG-16}     & MI-Attack & \cellcolor{palette3}  49.9 $\pm$  0.2 & \cellcolor{palette3}  50.0 $\pm$ 0.4 & \cellcolor{palette3}  46.7 $\pm$ 0.8 & \cellcolor{palette3} 50.4 $\pm$  1.2 & \cellcolor{palette3}  50.0 $\pm$ 0.6 \\
                            & DI-Attack & \cellcolor{palette3}  65.1 $\pm$ 0.2 & \cellcolor{palette3}  64.5 $\pm$ 0.4 & \cellcolor{palette3} 58.8$\pm$ 1.2  & \cellcolor{palette3}  64.1 $\pm$ 0.6 & \cellcolor{palette3}  60.9 $\pm$ 1.2\\
                            & TID-Attack & \cellcolor{palette4}  26.2 $\pm$ 1.2 & \cellcolor{palette4}  24.0 $\pm$ 1.2& \cellcolor{palette4}  13.0 $\pm$ 0.4 & \cellcolor{palette4}  20.8 $\pm$ 1.4 & \cellcolor{palette4}  18.8 $\pm$ 0.4 \\
                            & AEG (Ours) &\cellcolor{palette1} \textbf{97.5 $\pm$ 0.4} &\cellcolor{palette1} \textbf{96.1 $\pm$ 0.5} &\cellcolor{palette1} \textbf{85.2 $\pm$ 2.2} &\cellcolor{palette1} \textbf{94.1 $\pm$ 1.2}  &\cellcolor{palette1} \textbf{89.5 $\pm$ 1.3} \\
            \bottomrule
            \end{tabular}
            \caption{\small Error rates on $\mathcal{D}$ for average NoBox architecture transfer attacks with $\epsilon=0.03125$. The $\pm$ correspond to 2 standard deviations ($95.5\%$ confidence interval for normal distributions).\label{table:q2}}
    \end{small}
\end{table}

\subsection{NoBox Attacks Against Robust Classifiers}
\label{q3}
We now test the ability of our AEG framework to attack target models that have been robustified using adversarial and ensemble adversarial training \citep{madry2017towards,tramer2017ensemble}. 
For evaluation against PGD adversarial training, we use the public models as part of the MNIST and CIFAR-10 adversarial examples challenge.\footnote{\url{https://github.com/MadryLab/[x]_challenge}, for \texttt{[x]} in $\{\texttt{cifar10},\texttt{mnist}\}$.
Note that our threat model is more challenging than these challenges as we use non-robust source models.
} For ensemble adversarial training, we follow the approach of \citet{tramer2017ensemble} (see Appendix \ref{appendix:ens_adv_training}). 
We report our results in Table 3 and average the result of stochastic attacks over $5$ runs. 
We find that AEG achieves state-of-the-art performance in all settings, proving an average improvement in success rates of $54.1\%$ across all robustified MNIST models and $40.3\%$  on robustifi ed CIFAR-10 models.

\begin{table}[ht]
    \begin{small}
    \centering
            \begin{tabular}{cccccccc}
            \toprule
                   Dataset & Defence &  Clean & MI-Att$^\dag$ & DI-Att & TID-Att & SGM-Att $^\dag$  & AEG (Ours) \\
                            \midrule
\multirow{5}{*}{MNIST}      &  $\textrm{A}_\text{ens4}$ & $0.8$ & \cellcolor{palette3} $43.4$ & \cellcolor{palette3} $42.7$ & \cellcolor{palette4} $16.0 $ & N/A  &\cellcolor{palette1} \textbf{65.0}\\
                            & $\textrm{B}_\text{ens4}$ &  $0.7$ & \cellcolor{palette3} $20.7$ & \cellcolor{palette3} $22.8$ & \cellcolor{palette4} $8.5 $ & N/A &\cellcolor{palette1} \textbf{50.0} \\
                            & $\textrm{C}_\text{ens4}$ &  $0.8$ & \cellcolor{palette3} $73.8$ & \cellcolor{palette3} $30.0$ & \cellcolor{palette4} $9.5 $ & N/A &\cellcolor{palette1} \textbf{80.0}\\
                            & $\textrm{D}_\text{ens4}$ &  $1.8$ & \cellcolor{palette2} $84.4$ & \cellcolor{palette3} $76.0 $ & \cellcolor{palette2} $81.3 $ & N/A &\cellcolor{palette1} \textbf{86.7}\\
                            & Madry-Adv &  $0.8$ & \cellcolor{palette4} $2.0$ & \cellcolor{palette3} $3.1 $ & \cellcolor{palette4} $2.5$ & N/A &\cellcolor{palette1} \textbf{5.9}\\
                            \midrule
\multirow{5}{*}{CIFAR-10}   & $\textrm{RN-18}_\text{ens3}$ & $16.8$ & \cellcolor{palette4} $17.6$ & \cellcolor{palette4} $21.6$& \cellcolor{palette3} $33.1$ & \cellcolor{palette4} $19.9$ &\cellcolor{palette1} \textbf{52.2}\\
                            & $\textrm{WR}_\text{ens3}$ &  $12.8$ & \cellcolor{palette4} $18.4$ & \cellcolor{palette4} $ 20.6$ & \cellcolor{palette3} $28.8 $ & \cellcolor{palette4} $18.0$ &\cellcolor{palette1} \textbf{49.9} \\
                            & $\textrm{DN-121}_\text{ens3}$ &  $21.5$ & \cellcolor{palette4} $20.3$ & \cellcolor{palette3} $ 22.7$ & \cellcolor{palette3} $31.3$ & \cellcolor{palette3} $21.9$ &\cellcolor{palette1} \textbf{41.4} \\
                            & $\textrm{Inc-V3}_\text{ens3}$ &  $14.8$ & \cellcolor{palette4} $19.5$ & \cellcolor{palette2} $42.2$* & \cellcolor{palette3} $30.2*$ & \cellcolor{palette3} $35.5$* &\cellcolor{palette1} \textbf{47.5} \\
                            & Madry-Adv &  $12.9$ & \cellcolor{palette3} $17.2$  & \cellcolor{palette3} $16.6$ & \cellcolor{palette3} $16.6$ & \cellcolor{palette3} $16.0$ &\cellcolor{palette1} \textbf{21.6}\\
            \bottomrule
            \end{tabular}
            \caption{ \small\label{table:q3}
            Error rates on $\mathcal{D}$ for NoBox known architecture attacks against Adversarial Training and Ensemble Adversarial Training. ${}^*$ Attacks were done using WR. ${}^\dag$ Deterministic attack.}
    \end{small}
\end{table}

\section{Related Work}
In addition to non-interactive blackbox adversaries we compare against, there exists multiple hybrid approaches that combine crafting attacks on surrogate models which then serve as a good initialization point for queries to the target model \cite{papernot2016transferability, shi2019curls, huang2019black}. Other notable approaches to craft blackbox transfer attacks learning ghost networks \cite{li2018learning}, transforming whitebox gradients with small ResNets \cite{li2019regional}, and transferability properties of linear classifiers and 2-layer ReLu Networks \cite{charles2019geometric}. There is also a burgeoning literature of using parametric models to craft adversarial attacks such as the Adversarial Transformation Networks framework and its variants \cite{baluja2017adversarial,xiao2018generating}. Similar in spirit to our approach many attacks strategies benefit from employing a latent space to craft attacks \cite{zhao2017generating,tu2019autozoom,bose2019generalizable}. However, unlike our work, these strategies cannot be used to attack entire hypothesis classes.

Adversarial prediction games between a learner and a data generator have also been studied in the literature \citep{bruckner2012static}, and in certain situations correspond to a Stackelberg game \citet{bruckner2011stackelberg}. While similar in spirit, our theoretical framework is tailored towards crafting adversarial attacks against a fixed held out target model in the novel NoBox threat model and is a fundamentally different attack paradigm. Finally, \citet{erraqabi2018a3t}~ also investigate an adversarial game framework as a means for building robust representations in which an additional discriminator is trained to discriminate adversarial example from natural ones, based on the representation of the current classifier.


\section{Conclusion}
In this paper, we introduce the Adversarial Example Games (AEG) framework which provides a principled foundation for crafting adversarial attacks in the NoBox threat model. Our work sheds light on the existence of adversarial examples as a natural consequence of restricted entropy maximization under a hypothesis class and leads to an actionable strategy for attacking all functions taken from this class. 
Empirically, we observe that our approach leads to state-of-the-art results when generating attacks on MNIST and CIFAR-10 in a number of challenging NoBox attack settings. 
Our framework and results point to a promising new direction for theoretically-motivated adversarial frameworks. 
However, one major challenge is scaling up the AEG framework to larger datasets (e.g., ImageNet), which would involve addressing some of the inherent challenges of saddle point optimization ~\citep{Berard2020A}. 
Investigating the utility of the AEG framework for training robustified models is another natural direction for future work. 
\section*{Broader Impact}
\cut{
Adversarial attacks, especially ones under more realistic threat models, pose several important security, ethical, and privacy risks. In this work, we introduce the NoBox attack setting, which generalizes many other blackbox transfer settings, and we provide a novel framework to ground and study attacks theoretically and their transferability to other functions within a class of functions. As the NoBox threat model represents a more realistic setting for adversarial attacks, our research has the potential to be used against a class of machine learning models in the wild. 
In particular, in terms of risk,  malicious actors could use approaches based on our framework to generate attack vectors that compromise production ML systems or potentially bias them toward specific outcomes that may adversely affect certain population groups, such as minority groups. As a concrete example, one can consider creating transferrable examples in the physical world, such as the computer vision systems of autonomous cars. While prior works have shown the possibility of such adversarial examples ---i.e., adversarial traffic signs, we note that there is a significant gap in translating synthetic adversarial examples to adversarial examples that reside in the physical world~\citep{kurakin2016adversarial}. Understanding and analyzing the NoBox transferability of adversarial examples to the physical world---in order to provide public and academic visibility on these risks---is an essential direction for future research.

Based on the known risks of designing new kinds of adversarial attacks---discussed above---we now outline the ways in which our research is informed by the intent to mitigate these potential societal risks. For instance, our research demonstrates that one can successfully craft adversarial attacks even in the challenging NoBox setting. It raises many important considerations when developing robustness approaches. A straightforward extension is to consider our adversarial example game (AEG) framework as a tool for training robust models. On the theoretical side, exploring formal verification of neural networks against NoBox adversaries is an exciting direction for continued exploration. As an application, ML practitioners in the industry may choose to employ new forms of A/B testing with different types of adversarial examples, of which AEG is one method to robustify and stress test production systems further. Such an application falls in line with other general approaches to red teaming AI systems~ \cite{brundage2020toward} and verifiability in AI development. In essence, the goal of such approaches, including adversarial examples for robustness, is to align AI systems' failure modes to those found in human decision making. Another natural direction is the auditing of production ML systems by third parties, such as government or enforcement agencies.

Our work also has the potential to provide positive benefits in the realm of privacy, especially when concerned with the digital privacy of images. To combat this,~\citet{bose2018adversarial} showed typical face detectors, and by extension, facial recognition systems are vulnerable to adversarial attacks. Motivated by this use case, adversarial examples, including those presented in this work, can serve as a privacy-preserving mechanism by adding small amounts of crafted noise to safeguard against the unwanted loss of privacy to automatic facial recognition. 

Finally---as mentioned above---some studies show that the use of automatic facial recognition has a disproportionately negative impact on ethnic minorities, such as persons of color~\citep{raji2020saving}. Understanding the origin of model failures on adversarial examples pushes us towards the comprehension of such failure modes. Thus, it may also help explain why models disproportionately fail on people from ethnic minorities.
}

Adversarial attacks, especially ones under more realistic threat models, pose several important security, ethical, and privacy risks. In this work, we introduce the NoBox attack setting, which generalizes many other blackbox transfer settings, and we provide a novel framework to ground and study attacks theoretically and their transferability to other functions within a class of functions. As the NoBox threat model represents a more realistic setting for adversarial attacks, our research has the potential to be used against a class of machine learning models in the wild. In particular, in terms of risk, malicious actors could use approaches based on our framework to generate attack vectors that compromise production ML systems or potentially bias them toward specific outcomes.

As a concrete example, one can consider creating transferrable examples in the physical world, such as the computer vision systems of autonomous cars. While prior works have shown the possibility of such adversarial examples —i.e., adversarial traffic signs, we note that there is a significant gap in translating synthetic adversarial examples to adversarial examples that reside in the physical world [45]. Understanding and analyzing the NoBox transferability of adversarial examples to the physical world—in order to provide public and academic visibility on these risks—is an important direction for future research. 
Based on the known risks of designing new kinds of adversarial attacks—discussed above—we now outline the ways in which our research is informed by the intent to mitigate these potential societal risks. For instance, our research demonstrates that one can successfully craft adversarial attacks even in the challenging NoBox setting. It raises many important considerations when developing robustness approaches. A straightforward extension is to consider our adversarial example game (AEG) framework as a tool for training robust models. On the theoretical side, exploring formal verification of neural networks against NoBox adversaries is an exciting direction for continued exploration. As an application, ML practitioners in the industry may choose to employ new forms of A/B testing with different types of adversarial examples, of which AEG is one method to robustify and stress test production systems further. Such an application falls in line with other general approaches to red teaming AI systems [10] and verifiability in AI development. In essence, the goal of such approaches, including adversarial examples for robustness, is to align AI systems’ failure modes to those found in human decision making.

    \section*{Acknowledgments and Disclosure of Funding}

The authors would like to acknowledge Olivier Mastropietro, Chongli Qin and David Balduzzi for helpful discussions as well as Sebastian Lachapelle, Pouya Bashivan, Yanshuai Cao, Gavin Ding, Ioannis Mitliagkas, Nadeem Ward, and Damien Scieur for reviewing early drafts of this work.

\textbf{Funding.} This work is partially supported by the Canada CIFAR AI Chair Program (held at Mila), NSERC Discovery Grant RGPIN-2019-05123 (held by Will Hamilton at McGill), NSERC Discovery Grant RGPIN-2017-06936, an IVADO Fundamental Research Project grant PRF-2019-3583139727, and a Google Focused Research award (both held at U. Montreal by Simon Lacoste-Julien).
Joey Bose was also supported by an IVADO PhD fellowship, Gauthier Gidel by a Borealis AI fellowship and by the Canada Excellence Research Chair in "Data Science for Real-Time Decision-making" (held at Polytechnique by Andrea Lodi), and Andre Cianflone by a NSERC scholarship and a Borealis AI fellowship. 
Simon Lacoste-Julien and Pascal Vincent are
CIFAR Associate Fellows in the Learning in Machines \& Brains program. Finally, we thank Facebook for access to computational resources.

\textbf{Competing interests.} Joey Bose was formerly at FaceShield.ai which was acquired in 2020. W.L. Hamilton was formerly a Visiting Researcher at Facebook AI Research. Simon Lacoste-Julien additionally works part time as the head of the SAIT AI Lab, Montreal from Samsung.
\bibliography{bibliography}
\bibliographystyle{abbrvnat} 
\nosectionappendix
\appendix
\begin{toappendix}
\label{appendix}

\section{Experimental Details}
The experiments are subject to different sources of variations, in all our experiments we try to take into account those sources of variations when reporting the results. We detail the different sources of variations for each experiment and how we report them in the next section. 
\subsection {Source of variations}

\paragraph{NoBox attacks on a known architecture class} we created a random split of the MNIST and CIFAR10 dataset and trained a classifier on each splits. To evaluate each method we then use one of the classifiers as the source classifier and all the other classifiers as the targets we want to attack. We then compute the mean and standard deviation of the attack success rates across all target classifiers. To take into account the variability in the results that comes from using a specific classifier as the source model, we also repeat the evaluation by changing the source model. We report the average and $95\%$ interval (assuming the results follow a normal distribution) in Table~\ref{q1_table_new} by doing macro-averaging overall evaluations.

\paragraph{NoBox attacks across distinct architectures} For each architecture, we trained 10 different models. When evaluated against a specific architecture we evaluate against all models of this architecture. In Table 2, we report the mean and standard deviation of the error rates across all models.

\paragraph{NoBox Attacks against robust classifiers} For this experiment, we could only train a single robust target model per architecture because of our computational budget. The only source of variations is thus due to the inherent stochasticity of each method. Evaluating this source of randomness would require to run each method several times, unfortunately, this is quite expensive and our computational budget didn't allow for it. In Table 3, we thus only report a single number per architecture.

\section{Additional results}
\subsection{Quantitative Results}
\label{appendix:quant_results}
We now provide additional results in the form of whitebox and blackbox query attacks adapted to the NoBox evaluation protocol for Known-Architecture attacks which is the experimental setting in \textbf{Q1}. For whitebox attacks we evaluate APGD-CE and APGD-DLR \cite{croce2020reliable} which are improvements over the powerful PGD attack \cite{madry2017towards}. When ensembled with another powerful perturbation minimizing whitebox attack FAB \cite{croce2019minimally} and the query efficient blackbox Square attacks \cite{andriushchenko2019square} yields the current SOTA attack strategy called AutoAttack \cite{croce2020reliable} \cite{croce2020reliable}. Additionally, we compare with two parametric blackbox query approaches that both utilize a latent space in AutoZoom \cite{tu2019autozoom} and $\mathcal{N}$Attack \cite{li2019nattack}. To test transferability of whitebox and blackbox query attacks in the NoBox known architecture setting we give generous iteration and query budgets (10x the reported settings in the original papers) when attacking the source models, but only a single query for each target model. It is interesting to note that APGD variant whitebox attacks are significantly more effective than query based blackbox attacks but lack the same effectiveness of NoBox baselines. We hypothesize that the transferability of whitebox attacks may be due to the fact that different functions learn similar decision boundaries but different enough such that minimum distortion whitebox attacks such as FAB are ineffective.

\begin{table}[ht]
    \label{q1_table}
    \centering
          
            \begin{tabular}{cccc}
            \toprule
                   &    &   MNIST  & {CIFAR-10}\\
                    \midrule
\multirow{4}{*}{Whitebox} & AutoAttack* & $84.4 \pm 5.1$  & $91.0 \pm 1.9$ \\
                   & APGD-CE & $95.8 \pm 1.9 $ &  $97.5 \pm 0.7$ \\
                   & APGD-DLR & $83.9 \pm 5.4 $  & $90.7 \pm 2.1$ \\
                   & FAB & $5.4 \pm 2.2$  & $10.4 \pm 1.7$ \\
                   \midrule 
\multirow{3}{*}{Blackbox-query} & Square & $60.9 \pm 10.3$ & $21.9 \pm 2.8$ \\
                   &$\mathcal{N}$-Attack & $9.5 \pm 3.2 $  & $56.7 \pm 8.9$ \\
                   \midrule 
\multirow{4}{*}{\shortstack{Non-Interactive \\ Blackbox}}& MI-Attack & $93.7 \pm 1.1 $  & \textbf{99.9} $\pm 0.1$ \\
                    & DI-Attack & $95.9 \pm 1.6 $  & \textbf{99.9 $\pm$ 0.1}\\
                    & TID-Attack & $92.8 \pm 2.7$& $19.7 \pm 1.5$ \\
                    & SGM-Attack & N/A  & \textbf{99.8} $ \pm 0.3$ \\
                    & AEG (Ours) & \textbf{98.9 $\pm$ 1.4}  & $98.5 \pm 0.6$ \\
            \bottomrule
            \end{tabular}
            \caption{Test error rates for average blackbox transfer over architectures at $\epsilon=0.3$ for MNIST and $\epsilon=0.03125$ for CIFAR-10 (higher is better)}
    \end{table}

    \begin{table}[ht]
    \vspace{-1mm}
\begin{small}
    \centering
            \begin{tabular}{ccccccc}
            \toprule
                   Source & Attack &   VGG-16 & RN-18 & WR & DN-121 & Inc-V3 \\
                            \midrule
                            & Clean & $11.2 \pm 0.9$ & $13.1 \pm 2.0$ & $6.8 \pm 0.7$ & $11.2 \pm 1.4$ & $9.9 \pm 1.3$ \\
                            \midrule
\multirow{5}{*}{WR}      & MI-Attack & $67.8 \pm 3.01$ & $86.0 \pm 1.7$ & \textbf{99.9 $\pm$ 0.1} & $89.0 \pm 2.6$ & $88.2 \pm 1.4$\\
                            & DI-Attack & $68.3 \pm 2.4$ & 88.5 $\pm$ 2.1 & \textbf{99.9 $\pm$ 0.1} & 91.2 $\pm$ 1.6 & \textbf{91.5 $\pm$ 1.8}\\
                            & TID-Attack & $23.1 \pm 1.8 $ & $25.9 \pm 1.3$ & $20.6 \pm 1.0$ & $23.6 \pm 1.2$ & $21.9 \pm 1.7$ \\
                            & SGM-Attack & $69.1 \pm 2.1$ & 88.6 $\pm$ 2.0 & \textbf{99.6 $\pm$ 0.4} & 90.7 $\pm$ 1.9 & $86.8 \pm 2.2$ \\
                            & AEG (Ours) & $40.8\pm 3.22$ & $70.6 \pm 4.9$ & $98.5 \pm 0.6$ & \textbf{88.2 $\pm$ 4.6} & \textbf{89.6 $\pm$ 1.8} \\
                            & AEG (New) & \textbf{86.2 $\pm$ 1.8} & \textbf{94.1 $\pm$ 1.5} & 81.1 $\pm$ 1.1 & \textbf{93.1 $\pm$ 1.7} & 89.2 $\pm$ 1.8 \\
            \bottomrule
            \end{tabular}
            \caption{Error rates on $\mathcal{D}$ for average NoBox architecture transfer attacks with $\epsilon=0.03125$ with Wide-ResNet architecture}
    \end{small}
    \vspace{-2mm}
\end{table}
\subsection{Qualitative Results}
\label{appendix:qual_results}
As a sanity check we also provide some qualitative results about the generated adversarial attacks. In Figure~\ref{fig:mnist_adv} we show the 256 attacked samples generated by our method on MNIST. In Figure~\ref{fig:cifar_adv} we show on the left the 256 CIFAR samples to attack, and on the right the perturbations generated by our method amplified by a factor 10.

\begin{figure}[h]
    \centering
    \includegraphics[width=0.7\linewidth]{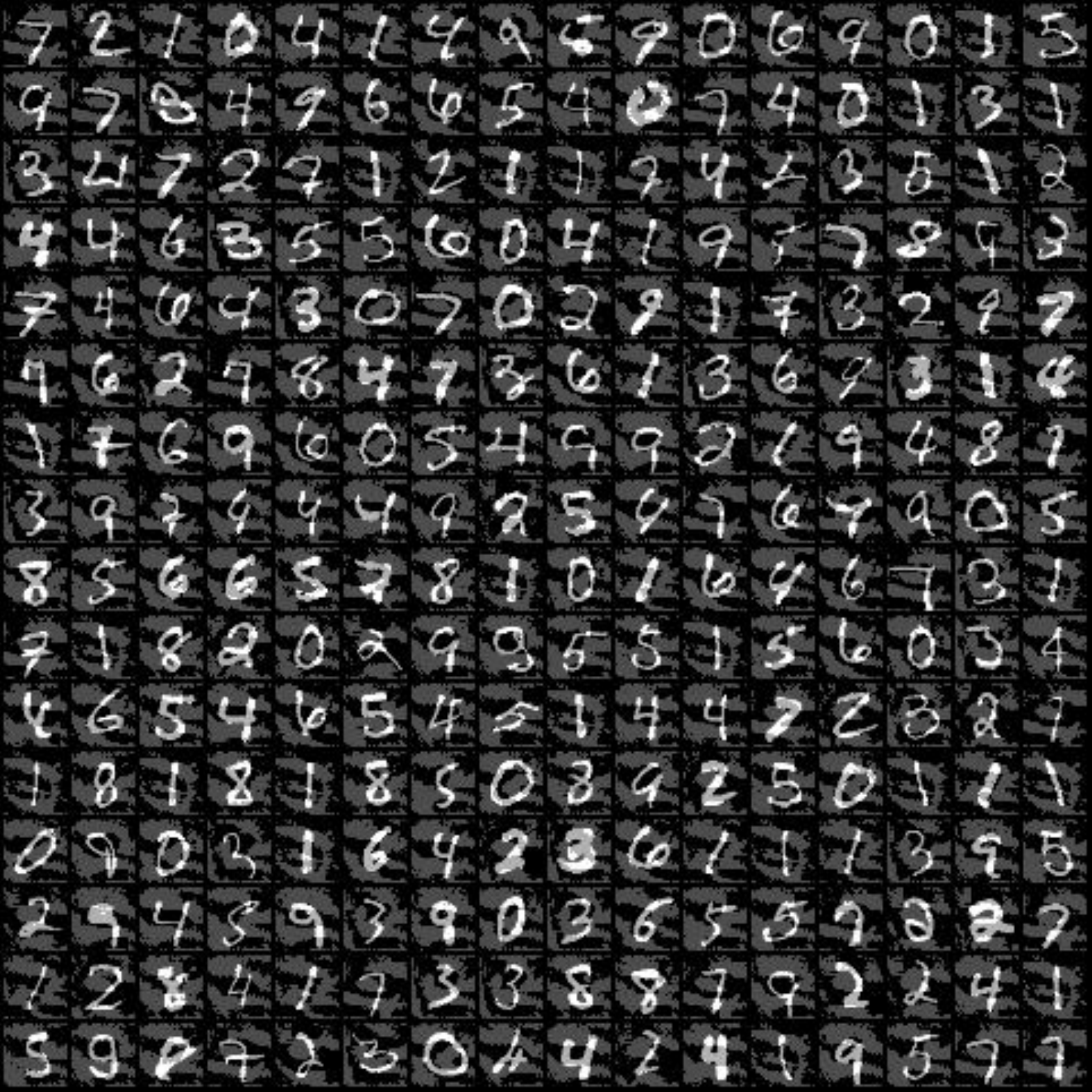}
    \caption{Attacks generated on MNIST by our method.}
    \label{fig:mnist_adv}
\end{figure}

\begin{figure}[h]
\centering
\begin{subfigure}
\centering
\includegraphics[width=.49\linewidth]{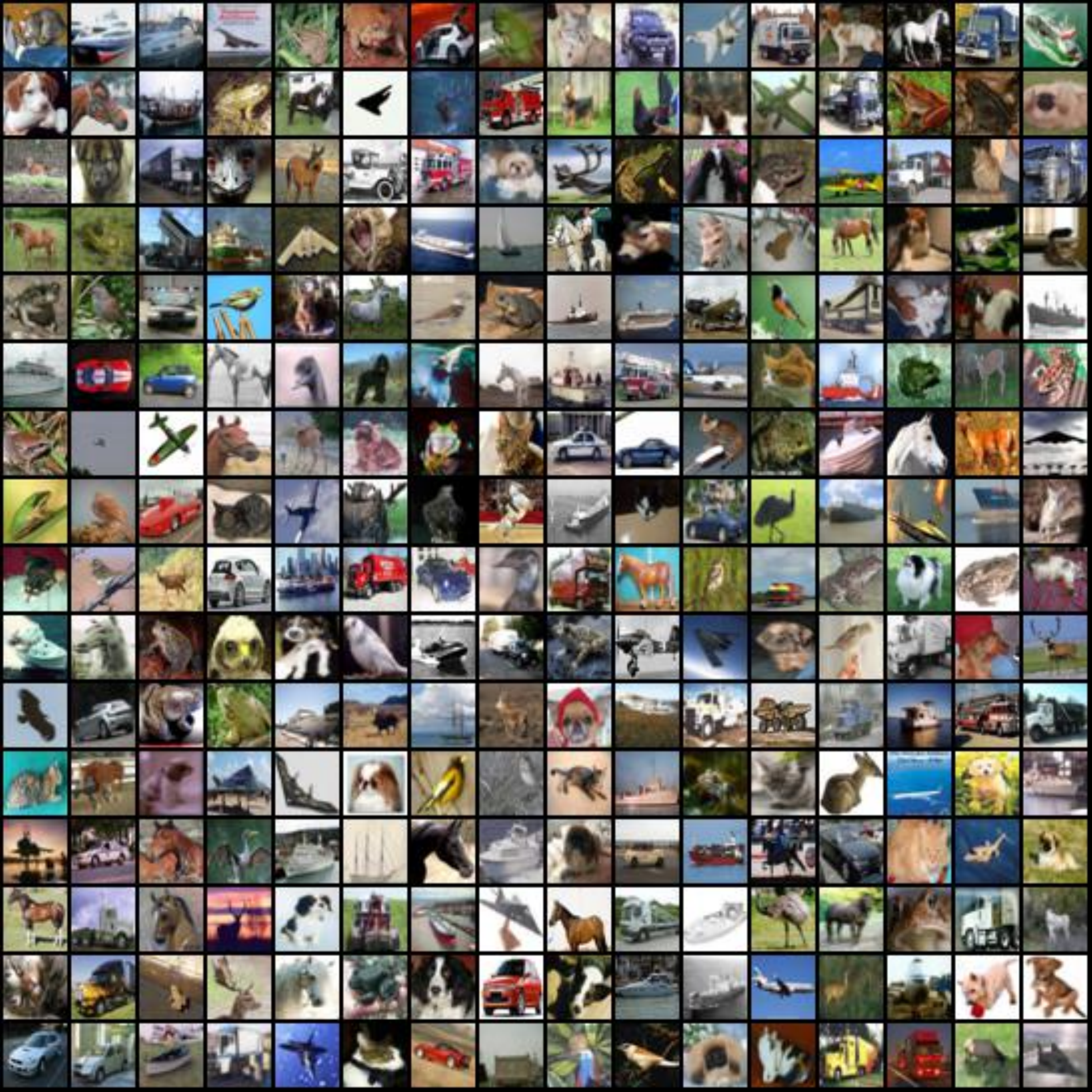}
\end{subfigure}
\begin{subfigure}
\centering
\includegraphics[width=.49\linewidth]{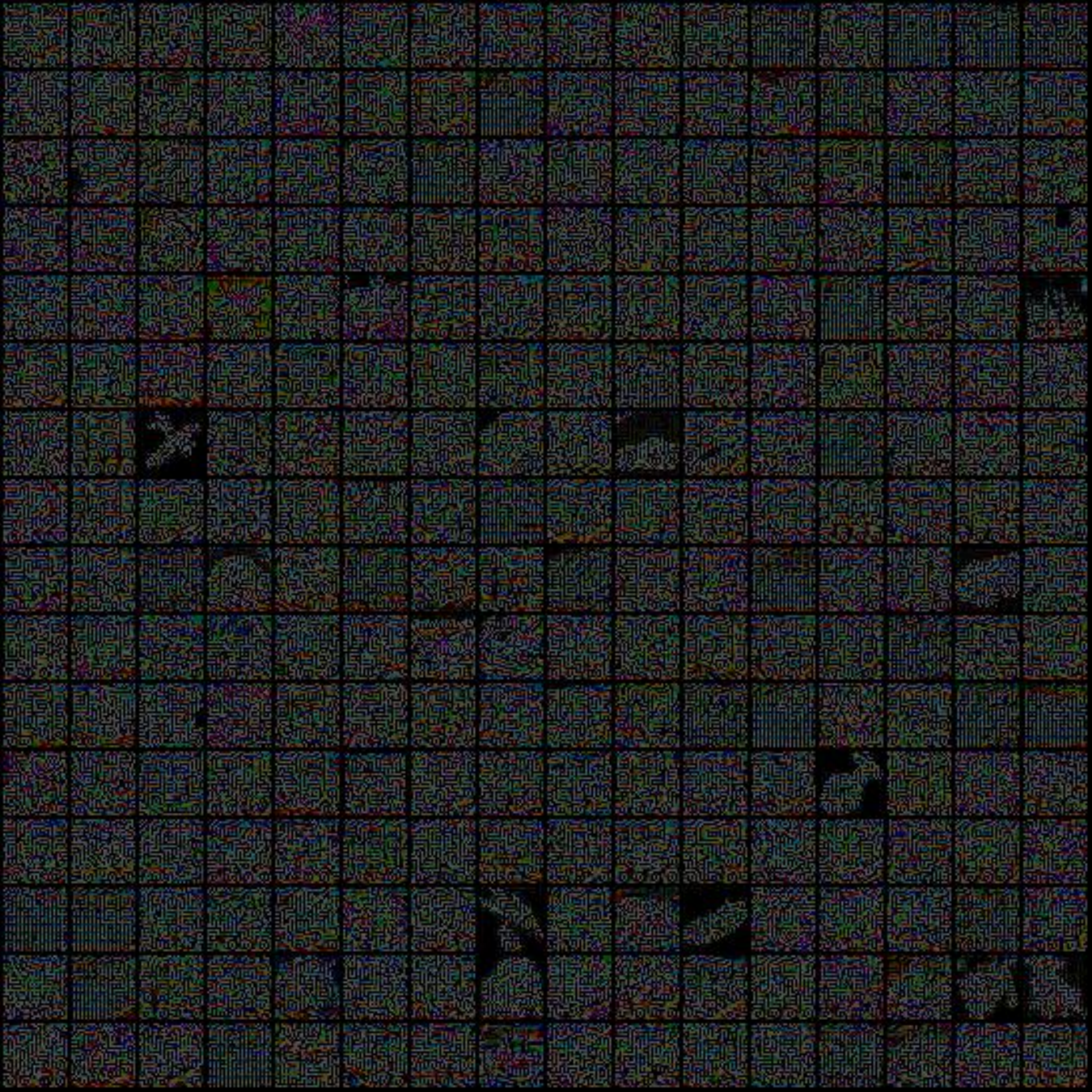}
\end{subfigure}
\caption{\textbf{Left:} CIFAR examples to attack. \textbf{Right:} Pertubations generated by our method amplified by a factor 10. An interesting observation is that the generator learns not to attack the pixel where the background is white.}
\label{fig:cifar_adv}
\end{figure}

\section{Implementation Details}
\label{app:architecture}

\cut{
\subsection{AEG Architecture}

\begin{figure}
    \centering
    \includegraphics[width= .9 \linewidth]{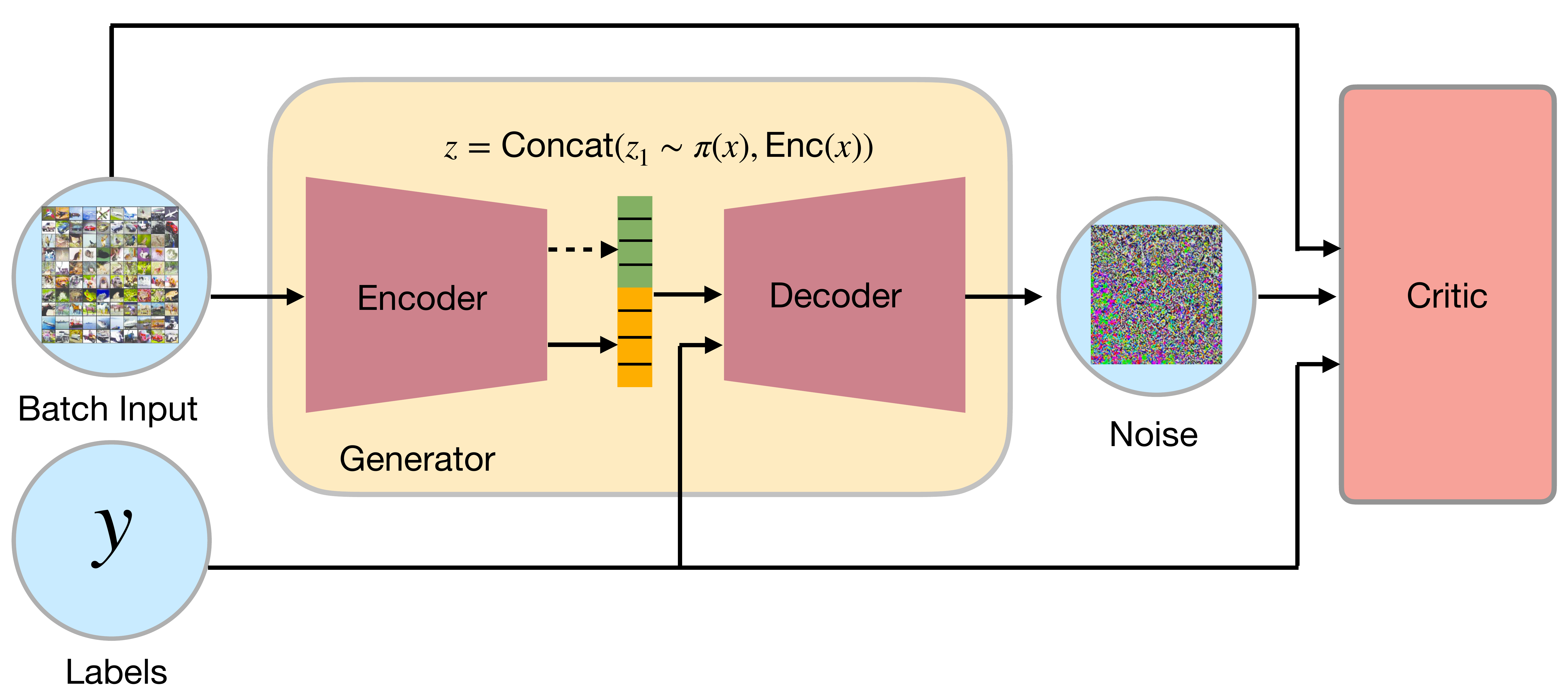}
    \caption{AEG framework architecture}
\label{fig:AEG_architecture}
\end{figure}
\cut{
\begin{wrapfigure}{r}{7.5cm}
\vspace{-4mm}
\centering
\includegraphics[width= \linewidth]{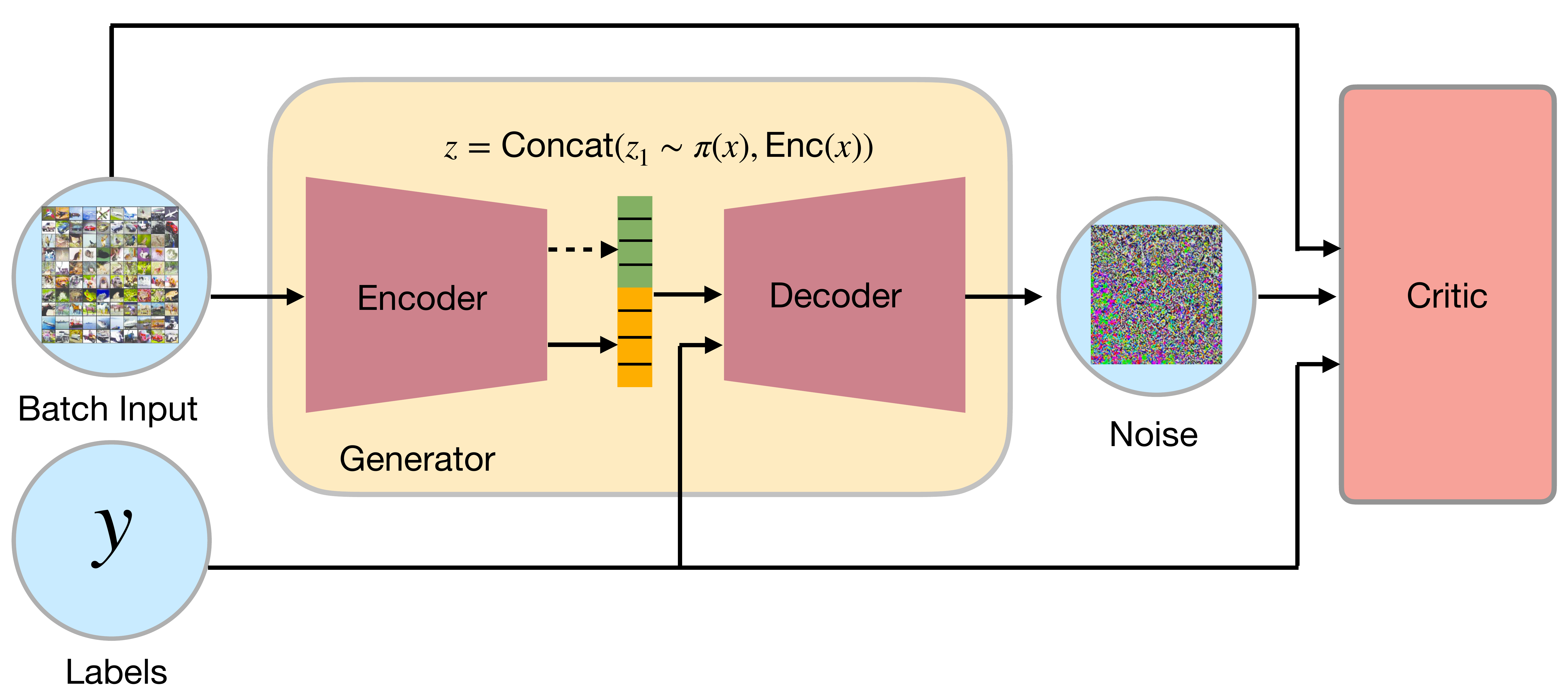}
\caption{AEG Architecture.}
\vspace{-2mm}
\end{wrapfigure}
}
}
We now provide additional details on training the representative classifiers and generators used in the AEG framework as outlined in \ref{fig:AEG_architecture}. We solve the game using the ExtraAdam optimizer~\citep{gidel2019variational} with a learning rate of $10^{-3}$. We allow the generator to update its parameters several times on the same batch of examples before updating the critic. In particular we update the generator until it is able to fool the critic or it reaches some fixed number of iterations. We set this max number of iterations to 20 in all our experiments. We also find that biasing the critic update various forms of adversarial training consistently leads to the most effective attack. We reconcile this phenomenon by noting that through adversarial training the critic itself becomes a robust model which provides a richer learning signal to the generator. Furthermore, the elegance of the AEG frameworks allows the practitioner to further bias the optimization process of the critic ---and consequently the generator--- through picking and choosing effective robustness techniques such as training with PGD adversarial examples generated at a prior timestep.

\subsection{Generator Architecture}

The architecture we used for the encoder and the decoder is described in Table~\ref{tab:resnet_arch} and~\ref{tab:conv_arch}. For MNIST we used a standard convolutional architecture and for CIFAR-10 we used a ResNet architecture. 
\begin{table}
    \centering
    \begin{tabular}{@{}c@{}}\toprule
\textbf{$d$-ResBlock}\\\toprule
	\multicolumn{1}{l}{\textit{Input}:$x$}\\
	\multicolumn{1}{l}{\textit{Forward for computing }$F(x)$:}\\
	Reflection pad (1) \\  
	conv. (ker: $3{\times}3$, $d \rightarrow d $; stride: $1$; pad: $1$) \\
    Batch Normalization \\ 
	ReLU \\
	Reflection pad (1) \\  
	conv. (ker: $3{\times}3$, $d \rightarrow d $; stride: $1$; pad: $1$) \\
	Batch Normalization \\
	\multicolumn{1}{l}{\textit{Output}:$x+F(x)$} \\
\bottomrule
\end{tabular}
    \caption{ResNet blocks used for the ResNet architectures (see Table~\ref{tab:resnet_arch}) for the Generator. Each ResNet block contains skip connection (bypass), and a sequence of convolutional layers, normalization, and the ReLU non--linearity.}
\label{tab:resblock}
\end{table}

\begin{table}\centering		 
\begin{small}
\begin{minipage}[b]{0.48\hsize}\centering  
\begin{tabular}{c}\toprule
\textbf{Encoder} \\\toprule
\textit{Input:} $x \in \sR^{3{\times}32{\times}32} $  \\ 
	Reflection Padding (3) \\  
    conv. (ker: $7{\times}7$, $32 \rightarrow 63$; stride: $1$; pad: $0$) \\
    Batch Normalization \\ 
	ReLU \\
	conv. (ker: $3{\times}3$, $63 \rightarrow 127$; stride: $2$; pad: $0$) \\
    Batch Normalization \\ 
	ReLU \\
	conv. (ker: $3{\times}3$, $127 \rightarrow 255$; stride: $2$; pad: $0$) \\
    Batch Normalization \\ 
	ReLU \\
	255-ResBlock \\
	255-ResBlock \\
	255-ResBlock \\
\bottomrule 
\end{tabular}
\end{minipage} \hfill
\begin{minipage}[b]{0.48\hsize}\centering  
\begin{tabular}{c}\toprule
\textbf{Decoder} \\\toprule
\textit{Input:} $(\psi(x),z,y)\in \sR^{256{\times}8{\times}8} $  \\ 
	256-ResBlock \\
	256-ResBlock \\
	256-ResBlock \\  
    Transp. conv. (ker: $3{\times}3$, $256 \rightarrow 128$; stride: $2$; pad: $0$) \\
    Batch Normalization \\ 
	ReLU \\
	Transp. conv. (ker: $3{\times}3$, $128 \rightarrow 64$; stride: $2$; pad: $0$) \\
    Batch Normalization \\ 
	ReLU \\
	ReflectionPadding(3) \\
	conv. (ker: $7{\times}7$, $64 \rightarrow 32$; stride: $1$; pad: $0$) \\
    Tanh \\
\bottomrule 
\end{tabular}
\end{minipage}
\end{small}
\caption{
Encoder and Decoder for the convolutional generator used for the MNIST dataset.}
\label{tab:resnet_arch}
\end{table}

\begin{table}\centering		 
\begin{small}
\begin{minipage}[b]{0.48\hsize}\centering  
\begin{tabular}{c}\toprule
\textbf{Encoder} \\\toprule
\textit{Input:} $x \in \sR^{28{\times}28} $  \\ 
    conv. (ker: $3{\times}3$, $1 \rightarrow 64$; stride: $3$; pad: $1$) \\
    LeakyReLU($0.2$) \\ 
	Max Pooling (stride: $2$) \\
	conv. (ker: $3{\times}3$, $64 \rightarrow 32$; stride: $2$; pad: $1$) \\
    LeakyReLU($0.2$) \\ 
	Max Pooling (stride: $2$) \\
\bottomrule 
\end{tabular}
\end{minipage} \hfill
\begin{minipage}[b]{0.48\hsize}\centering  
\begin{tabular}{c}\toprule
\textbf{Decoder} \\\toprule
\textit{Input:} $(\psi(x),z,y)\in \sR^{64{\times}2{\times}2} $   \\
    Transp. conv. (ker: $3{\times}3$, $64 \rightarrow 32$; stride: $2$; pad: $1$) \\
    LeakyReLU($0.2$) \\ 
	Max Pooling stride $2$ \\
	Reflection Padding (3) \\  
    Transp. conv. (ker: $5{\times}5$, $32 \rightarrow 16$; stride: $3$; pad: $1$) \\
    LeakyReLU($0.2$) \\ 
	Max Pooling stride $2$ \\
	Transp. conv. (ker: $2{\times}2$, $16 \rightarrow 1$; stride: $2$; pad: $1$) \\
    Tanh \\
\bottomrule 
\end{tabular}
\end{minipage}
\end{small}
\caption{
Encoder and Decoder for the ResNet generator used for the MNIST dataset.}
\label{tab:conv_arch}
\end{table}

\subsection{Baseline Implementation Details}
The principal baselines used in the main paper include the Momentum-Iterative Attack (MI-Attack)~\citep{dong2018boosting}, the Input Diversity (DI-Attack)~\citep{xie2019improving}, the Translation-Invariant (TID-Attack)~\citep{dong2019evading} and the Skip Gradient Method (SGM-Attack)~\citep{wu2020skip}. As Input Diversity and Translation invariant are approaches that generally can be combined with existing attack strategies we choose to use the powerful Momentum-Iterative attack as our base attack. Thus the DI-Attack consists of random input transformations when using an MI-Attack adversary while the TID-attack further adds a convolutional kernel ontop of the DI-Attack. We base our implementions using the AdverTorch~\citep{ding2019advertorch} library and adapt all baselines to this framework using original implementations where available.
In particular, when possible we reused open source code in the Pytorch library~\citep{paszke2019pytorch} otherwise we re-implement existing algorithms. We also inherit most hyperparameters settings when reporting baseline results except for number steps used in iterated attacks. We find that most iterated attacks benefit from additional optimization steps when attacking MNIST and CIFAR-10 classifiers. Specifically, we allot a $100$ step budget for all iterated attacks which is often a five to ten fold increase than the reported setting in all baselines. 

\subsection{Ensemble Adversarial Training Architectures}
\label{appendix:ens_adv_training}
We ensemble adversarially train our models in accordance with the training protocol outlined in \cite{tramer2017ensemble}.
For MNIST models we train a standard model for 6 epochs, and an ensemble adversarial model using adversarial examples from the remaining three architectures for 12 epochs. The specific architectures for Models A-D are provided in Table. 8. Similarly, for CIFAR-10 we train both the standard model and ensemble adversarial models for $50$ epochs. For computationally efficiency we randomly sample two out of three held out architectures when ensemble adversarially training the source model.

\begin{table}[ht]
    \label{appendix:mnist_ens_adv_training_table}
    \centering
          
            \begin{tabular}{ccccc}
            \toprule
                   &  A  & B  & C & D \\
                    \midrule
                   & Conv(64, 5, 5) + Relu & Dropout(0.2) & Conv(128, 3, 3) + Tanh & \multirow{2}{*}{\shortstack{FC(300) + Relu \\ Dropout(0.5)}}\\
                   & Conv(64, 5, 5) + Relu &  Conv(64, 8, 8) + Relu &  MaxPool(2,2) & \\
                   & Dropout(0.25) & Conv(128, 6, 6) + Relu  & Conv(64, 3, 3) + Tanh & \multirow{2}{*}{\shortstack{FC(300) + Relu \\ Dropout(0.5)}} \\
                   & FC(128) + Relu & Conv(128, 6, 6) + Relu & MaxPool(2,2) & \\
                   & Dropout(0.5) &  Dropout(0.5)  &  FC(128) + Relu &  \multirow{2}{*}{\shortstack{FC(300) + Relu \\ Dropout(0.5)}}\\
                   & FC + Softmax &  FC + Softmax &  FC + Softmax & \\
                   & & & & \multirow{2}{*}{\shortstack{FC(300) + Relu \\ Dropout(0.5)}} \\
                   & & & & \\
                   & & & & FC + Softmax \\
            \bottomrule
            \end{tabular}
            \caption{MNIST Ensemble Adversarial Training Architectures)}
    \end{table}

\section{Further Related Work}
Adversarial attacks can be classified under different threat models,
which impose different access and resource restrictions on the attacker \cite{akhtar2018threat}. The whitebox setting, where the attacker has full access to the model parameters and outputs, thus allowing the attacker to utilize gradients based methods to solve a constrained optimization procedure. 
This setting is more permissive than the semi-whtiebox and the blackbox setting, the latter of which the attacker has only access to the prediction \cite{papernot2016transferability,papernot2017practical} or sometimes the predicted confidence~\citep{guo2019simple}. In this paper, we focus on a challenging variant of the conventional blackbox threat model which we call the NoBox setting which further restricts the attacker by \emph{not} allowing any query from the target model. While there exists a vast literature of adversarial attacks, we focus on ones that are most related to our setting and direct the interested reader to comprehensive surveys for adversarial attacks and blackbox adversarial attacks \cite{bhambri2019survey,chakraborty2018adversarial}.

\xhdr{Whitebox Attacks}
The most common threat model for whitebox adversarial examples are $l_p$-norm attacks, where $p \in \{2, \infty \}$ is the choice of norm ball used to define the attack budget. One of the earliest gradient based attacks is the Fast Gradient Sign Method (FGSM) \cite{goodfellow2014explaining}, which computes bounded perturbations in a single step by computing the signed gradient of the loss function with respect to a clean input. More powerful adversaries can be computed using multi-step attacks such as DeepFool \cite{moosavi2016deepfool} which iteratively finds the minimum distance over perturbation direction needed to cross a decision boundary. For constrained optimization problems the Carlini-Wagner (CW) attack \cite{carlini2017towards} is a powerful iterative optimization scheme which introduces an attack objective designed to maximize the distance between the target class and the most likely adversarial class. Similarly, projected gradient
descent based attacks has been shown to be the strongest class of adversaries for $l_2$ and $l_{\infty}$ norm attacks \cite{madry2017towards} and even provides a natural way of robustifying models through adversarial training. Extensions of PGD that fix failures due to suboptimal step size and
problems of the objective function include AutoPGD-CE (APGD-CE) and AutoPGD-DLR (APGD-DLR) and leads to the state of the art whitebox attack in AutoAttack \cite{croce2020reliable} which ensembles two other strong diverse and parameter free attacks.

\xhdr{Blackbox Attacks}
Like whitebox attacks the adversarial goal for a blackbox attacker remains the same with the most common threat model also being $l_p$ norm attacks. Unlike, whitebox attacks the adversarial capabilities of the attacker is severely restricted rendering exact gradient computation impossible. In lieu of exact gradients, early blackbox attacks generated adversarial examples on surrogate models in combination with queries to the target model \cite{papernot2016transferability}. When given a query budget gradient estimation is an attractive approach with notable approaches utilizing black box optimization schemes such as Finite Differences \cite{chen2017zoo}, Natural Evolutionary Strategies \cite{ilyas2018black, jiang2019black}, learned priors in a bandit optimization framework \cite{ilyas2018prior}, meta-learning attack patterns \cite{du2019query}, and query efficient. 

\xhdr{Defenses}
In order to protect against the security risk posed by adversarial examples there have been many proposed defense strategies. Here we provide a non-exhaustive list of such methods. Broadly speaking, most defense approaches can be categorized into either robust optimization techniques, gradient obfuscation methods, or adversarial example detection algorithms \cite{xu2019adversarial}. Robust optimization techniques aim to improve the robustness of a classifier by learning model parameters by incorporating adversarial examples from a given attack into the training process \cite{madry2017towards,tramer2017ensemble,Ding2020MMA}. On the other hand obfuscation methods rely on masking the input gradient needed by an attacker to construct adversarial examples \cite{song2017pixeldefend, buckman2018thermometer, guo2017countering, dhillon2018stochastic}. 

In adversarial example detection schemes the defender seeks to sanitize the inputs to the target model by rejecting any it deems adversarial. Often this involves training auxiliary classifiers or differences in statistics between adversarial examples and clean data \cite{grosse2017statistical, metzen2017detecting, gong2017adversarial}.

\end{toappendix}

\end{document}